\newif\ifAIRecArxivStyle
\AIRecArxivStyletrue

\ifAIRecArxivStyle
\documentclass[11pt]{article}
\else
\documentclass[mnsc,sglanonrev]{informs4_WP}
\fi

\ifAIRecArxivStyle
\usepackage{ai-recommendations-arxiv}
\else
\usepackage{ai-recommendations-informs}
\fi

\newcommand{\AIRecPaperTitle}{Strategic Advice in the Age of Personal AI}
\newcommand{\AIRecPaperKeywords}{personal AI; decision support; advice taking; human-AI collaboration; recommendation design; trust; behavioral operations; forecasting experiments}
\newcommand{\AIRecPaperAbstract}{Personal AI assistants are changing how individuals use advice. We study how an advisor should design its recommendation in anticipation of stochastic consultation with personal AI whose recommendation is predictable. Personal AI enters through two dimensions: consultation probability and relative trust, which captures the relative influence personal AI receives when consulted. In the baseline model, the advisor optimally counteracts the personal AI signal. Counteraction increases with consultation probability but is hump-shaped in relative trust. The advisor’s minimized loss is hump-shaped in consultation probability, vanishing when personal AI is never or always consulted. Greater relative trust in personal AI increases the irreducible loss arising from stochastic consultation. We extend the analysis to partial predictability and costly recommendation adjustment, characterizing their effects on optimal recommendations and minimized loss. The framework also accommodates richer information structures, including settings in which personal AI is perceived as having access to private information relevant to the task. We introduce an online forecasting experiment that examines how participants obtain personal AI advice and combine it with an advisor’s recommendation and their initial judgments. Participants place weight on all three inputs. When access requires an additional action, some participants do not seek personal AI advice, while some others attempt to obtain it without success. Together, these findings highlight two distinct aspects of personal AI use: whether advice is obtained and how much weight it receives when available.
}

\ifAIRecArxivStyle
\title{\AIRecPaperTitle}
\author[1]{Yueyang Liu}
\affil[1]{Jones Graduate School of Business, Rice University}
\author[2]{Wichinpong Park Sinchaisri}
\affil[2]{Haas School of Business, University of California, Berkeley}
\affil[ ]{\textit {yueyl@rice.edu, parksinchaisri@berkeley.edu}}
\else
\RUNAUTHOR{Liu and Sinchaisri}
\RUNTITLE{\AIRecPaperTitle}

\TITLE{\AIRecPaperTitle}
\ARTICLEAUTHORS{%
\AUTHOR{Yueyang Liu}
\AFF{Jones Graduate School of Business, Rice University, \EMAIL{yueyl@rice.edu}}

\AUTHOR{Wichinpong Park Sinchaisri}
\AFF{Haas School of Business, University of California, Berkeley, \EMAIL{parksinchaisri@berkeley.edu}}
}

\ABSTRACT{\AIRecPaperAbstract}

\fi

\begin{document}

\AIRecBeginMainBibunit

\maketitle

\ifAIRecArxivStyle
\begin{abstract}
\AIRecPaperAbstract
\end{abstract}
\else
\vspace{-0.6cm}
\fi

\section{Introduction}\label{sec:introduction}

Personal AI assistants are becoming a
routine part of how individuals seek and use advice at work and in everyday
life \citep{brynjolfsson2025generative,chatterji2025people,miyazaki2024public}.
At work, an employee may receive a recommendation from an enterprise
decision-support system and then ask a personal AI assistant for a second
opinion. An individual may also receive professional advice about a financial,
medical, or legal decision and then consult ChatGPT, Claude, or Gemini. To
design effective recommendations, advisors must anticipate not only the extent
to which individuals will follow their advice
\citep{bastani2026improving,caro2023believing,dietvorst2018overcoming,grand2026best,ibanez2018discretionary,sun2022predicting,van2010ordering},
but also how they will combine it with advice from personal AI.

Two features jointly shape the advisor’s strategic environment with personal AI. First, consultation is stochastic. Individuals do not always consult personal AI, and an advisor typically does not know ex ante whether its recommendation will be evaluated on its own or aggregated with a personal AI signal. Second, personal AI recommendations are often predictable. Whereas a human expert’s advice may depend on private information or individual judgment, the output of a widely available AI system can often be anticipated, approximated, or even reproduced by querying the same tool. Together, these features create a strategic setting in which the advisor faces an outside source whose content may be predictable, but whose use is uncertain.

We develop a theoretical framework to study how an advisor should design its recommendation in such a setting. A decision maker seeks to estimate an unknown parameter $\theta$. The decision maker receives a recommendation $S_A$ from the advisor and, with probability $p$, also consults personal AI and receives a recommendation $S_P$. The decision maker combines prior beliefs with the available recommendations according to their perceived precision, or trust, to form a point estimate of $\theta$. Anticipating this decision rule, the advisor chooses a single recommendation before knowing whether personal AI will be consulted. The advisor seeks to steer the final decision toward a target $r$, which may represent the true state or an institutional objective.

We characterize the advisor's optimal recommendation and minimized loss, focusing on how they vary with two dimensions of personal AI use: the consultation probability $p$ and the relative trust $t$, the latter of which measures the influence of personal AI relative to the combined influence of the advisor and the decision maker's prior beliefs.

The advisor optimally counteracts the personal AI signal, with counteraction intensity increasing in consultation probability $p$ but hump-shaped in relative trust $t$. A higher consultation probability strengthens the incentive to offset personal AI's influence. The effect of relative trust is less straightforward: little counteraction is needed when personal AI has little influence, while counteraction becomes ineffective when its influence is overwhelming. Counteraction is therefore strongest at an intermediate level of relative trust.

The advisor's minimized loss depends on the disagreement between the personal AI recommendation and the advisor's target and is hump-shaped in $p$. The advisor can adjust its recommendation to offset the influence of the prior because that influence is predictable. Similarly, when personal AI is never consulted or always consulted, consultation status is certain, so the advisor can align the final decision with its target and loss vanishes. When consultation is uncertain, however, a single recommendation generally cannot align the final decision with the target in both states. This creates an irreducible loss that peaks at an intermediate consultation probability.

The minimized loss depends on trust only through the relative trust $t$ and increases with $t$. Greater relative trust in personal AI widens the gap between the recommendations that would be optimal with and without consultation. This increases the irreducible loss arising from stochastic consultation.

We extend the baseline analysis to partially predictable personal AI recommendations and costly recommendation adjustment. When the advisor’s information only partially predicts the personal AI recommendation, the advisor counteracts its predictable component, while the unpredictable component generates an additional loss. When deviating from a reference recommendation incurs a cost, the optimal recommendation lies between the costless optimum and the reference recommendation, moving toward the reference as adjustment costs increase.

The framework also accommodates richer information structures, such as settings in which the decision maker perceives personal AI as having access to private, user-specific information relevant to the task. The decision rule and the baseline characterizations of optimal recommendations and minimized loss continue to apply with appropriately adjusted perceived precisions and relative trust. As the perceived importance of private information increases, personal AI receives greater relative trust, increasing the advisor’s minimized loss.

We complement the analysis with the City Manager Game, an online forecasting experiment involving 69 participants. In each round, participants make an initial forecast about a municipal task, receive controlled numerical advice, and submit a final forecast. Advice is presented as coming from enterprise AI, personal AI, or a human expert. The experiment examines reliance on individual and competing sources, consultation when personal-AI access requires an additional action, and the performance of a fixed counteraction policy. In rounds with both AI sources, the fitted aggregation rule places weight on each recommendation and retains roughly one-third of the weight on the initial judgment. In rounds requiring access, participants differ in whether they seek personal-AI advice, and an attempt does not always result in obtaining a recommendation. The evidence thus distinguishes the use of displayed advice from the decision and ability to obtain it.

We implement a naïve counteraction policy that differs from the theoretically optimal recommendation. It exactly offsets personal AI’s error when participants place equal weight on the two recommendations and no weight on their initial judgments. In our sample, the policy lowers average final forecast error relative to unadjusted advice, but error remains. A better-performing policy may require estimating and accounting for the weights on initial judgments and each recommendation, as our theoretical analysis suggests.

\section{Related Literature}
\label{sec:related-literature}

Our paper contributes to research on advice taking, algorithmic decision support, and human-AI
collaboration. Across these literatures, a central insight is that the value of advice depends not only on its statistical quality, but also on how the decision maker incorporates it into the final decision. We study a setting in which this incorporation problem is complicated by personal AI: the decision maker may receive a recommendation from a focal advisor, but may also consult an outside AI assistant before acting.

A foundational stream of behavioral research on advice taking is the Judge--Advisor System \citep{harvey1997taking,yaniv2004receiving,bonaccio2006advice,soll2009strategies,bailey2023meta}. This stream of work documents systematic patterns in advice aggregation, including egocentric discounting and its variation across contexts. Decision makers often revise their initial judgments only partially after receiving advice, and the extent of discounting varies with their prior knowledge, uncertainty, perceived expertise of the advisor, and task stakes. While this literature carefully characterizes how individuals integrate advice given fixed information sources, it typically treats advisors as passive and non-strategic. It does, however, provide experimental support for modeling advice aggregation as linear weighted averaging; in our framework, these weights correspond to perceived precision or trust in each source.

A growing operations literature studies human-aware decision support, often from the perspective of a focal algorithmic advisor whose recommendations are only partially followed by human decision makers. Field and experimental evidence shows that managers and workers frequently deviate from algorithmic recommendations, and that these deviations depend on discretion, beliefs, incentives, and task context \citep{van2010ordering,ibanez2018discretionary,caro2023believing}. Recent work further argues that decision-support systems should account for predictable human responses rather than optimize only for full compliance \citep{sun2022predicting,grand2026best,bastani2026improving}. Building on this insight, we model the advisor as a strategic actor with a target $r$, rather than as a passive information provider. The recommendation itself becomes a strategic instrument that tries to shift the decision maker's final decision toward the advisor's objective. Our setting also introduces a different source of deviation: the decision maker may not merely underweight or depart from the advisor's recommendation, but may combine it with personal AI. This changes the advisor's problem from designing recommendations for partial adherence to designing recommendations under possible aggregation with a competing and predictable outside signal.

A related literature studies trust and reliance on algorithmic advice. Prior work documents both algorithm aversion, in which users reduce reliance after observing algorithmic errors, and algorithm appreciation, in which users prefer algorithmic advice under some conditions \citep{dietvorst2015algorithm,dietvorst2018overcoming,logg2019algorithm}. Subsequent research examines mechanisms that shape trust and reliance, including interface cues, explanations, prior experience, the ability to modify algorithmic outputs, and the interpretability of recommendations \citep{parasuraman1997humans,lee2004trust,bansal2021does,naiseh2023different,shin2021effects,dietvorst2018overcoming,bastani2026improving}. We build on this literature by representing trust as perceived precision in the decision maker's decision rule. Our contribution is to show that trust matters not only because it changes reliance on the focal advisor, but also because it determines the advisor's exposure to a competing outside AI signal. In our setting, the relevant object is therefore relative influence: how much weight personal AI receives compared with the advisor and the decision maker's own prior judgment.

Our model also relates more broadly to forecast aggregation and strategic communication. Forecast-aggregation research studies how multiple forecasts or signals should be combined when information sources differ in precision or contain overlapping information \citep{arieli2018robust,palley2019extracting,steyvers2022bayesian}. We take a different approach: rather than solving for the optimal aggregation rule, we take the decision maker's aggregation behavior as a primitive and study how a focal advisor should design its recommendation in anticipation of that behavior. Strategic communication research studies how an informed or biased sender communicates with a receiver whose action depends on the message, either through nonbinding cheap talk \citep{crawford1982strategic} or through committed information design in Bayesian persuasion \citep{kamenica2011bayesian}. Our setting differs because the decision maker does not strategically decode the advisor's message or respond to a designed information structure. Instead, she treats recommendations as informative signals with perceived precision, while the advisor anticipates this behavioral aggregation rule. This formulation allows us to study recommendation design when the same advice may be interpreted either alone or jointly with a predictable personal-AI signal.

Finally, our paper contributes to emerging work on generative AI. Recent empirical research documents how generative AI affects productivity, labor-market outcomes, task execution, and the use of expertise \citep{brynjolfsson2025generative,demirci2025ai,hui2024short}, while related work highlights the organizational consequences of AI for authority, accountability, and how users interpret AI output \citep{allen2022algorithm,anthony2023collaborating,leonardi2026knowing}. We complement this work by developing a theoretical model of personal-AI adoption as a strategic environment for a focal advisor. In our setting, personal AI is selectively consulted but its recommendation is predictable, so the advisor competes with an outside signal that is uncertain in use but strategically salient in content. We decompose personal-AI adoption into an extensive margin, the probability of consultation, and an intensive margin, the relative influence or trust placed on personal AI conditional on consultation. This decomposition allows us to characterize how personal AI changes optimal recommendation design and advisor vulnerability.

\section{Model}\label{sec:model}

Consider a decision maker who seeks to estimate an unknown parameter $\theta$. The decision maker holds a prior belief $\theta \sim \mathcal{N}(\mu_0,\sigma_0^2)$. The decision maker receives a recommendation from a focal advisor, denoted by $S_A$, and believes that the advisor's recommendation is a noisy observation of $\theta$ given by
\[
S_A = \theta + \epsilon_A, \qquad \epsilon_A \sim \mathcal{N}(0,\sigma_A^2).
\]

With probability $p\in[0,1]$, the decision maker also consults a personal AI and receives a recommendation $S_P$, which she believes is given by
\[
S_P = \theta + \epsilon_P, \qquad \epsilon_P \sim \mathcal{N}(0,\sigma_P^2).
\]

After observing the available recommendations, the decision maker forms a final decision $D$ equal to the posterior mean of $\theta$ under these beliefs.

Define
\[
r_A = \frac{\sigma_0^2}{\sigma_A^2}, \qquad
r_P = \frac{\sigma_0^2}{\sigma_P^2}, \qquad
t = \frac{r_P}{1+r_A}.
\]
The parameters $r_A$ and $r_P$ capture the perceived precision of the advisor and personal AI relative to the decision maker's prior. The parameter $t$ is the relative trust ratio. The equations above describe the decision maker's belief structure and do not impose how the recommendations are actually generated. In particular, the advisor may choose $S_A$ strategically.

The advisor seeks to make the decision maker's final decision $D$ as close as possible to a target $r$. Formally, the advisor chooses its recommendation $S_A$ to minimize
\[
\min_{S_A} \; E\!\left[(D-r)^2\right].
\]
The target $r$ may represent the ground truth, an imperfect estimate of the ground truth, or it may reflect an institutional or strategic objective.

In the baseline model, the advisor knows $p$, $\mu_0$, $r_A$, and $r_P$, and observes or can perfectly predict the personal AI recommendation $S_P$ before choosing $S_A$. The advisor chooses $S_A$ before knowing whether the decision maker will in fact consult personal AI. We later relax the full-predictability assumption and allow $S_P$ to be only partially predictable.

\section{Advisor Response to Personal AI Adoption}\label{sec:advisor-response}

This section characterizes the advisor's strategic response to personal AI adoption. We begin by deriving the decision maker's decision rule as a function of the available recommendations, highlighting the extensive and intensive margins through which personal AI affects the advisor's problem. We then characterize the advisor's optimal recommendation and show how it decomposes into target alignment, correction for the decision maker's prior, and strategic counteraction against the personal-AI recommendation. We next study how counteraction intensity and the advisor's resulting loss vary with personal-AI consultation probability $p$ and relative trust $t$. Finally, we relax the baseline assumptions by allowing the personal-AI recommendation to be only partially predictable and by introducing a cost of adjusting the advisor's recommendation.

\subsection{Decision Maker Aggregation}

The decision maker forms the final decision through Bayesian updating, and personal AI affects the decision through two margins. The extensive margin is adoption, captured by $p$: the personal-AI recommendation enters the decision only when personal AI is consulted. The intensive margin is relative trust, captured by $t$: conditional on consultation, it captures the influence of personal AI relative to the combined weight of the advisor and the decision maker's prior. Together, these two margins define the strategic environment the advisor must navigate.

If personal AI is not consulted, the decision maker's final decision is
\[
D_0(S_A)=\frac{\mu_0+r_A S_A}{1+r_A}.
\]
If personal AI is consulted, the final decision is
\[
D_1(S_A,S_P)=\frac{\mu_0+r_A S_A+r_P S_P}{1+r_A+r_P}.
\]
These two decision rules together show how personal AI changes the decision maker's final decision. When personal AI is consulted, its recommendation enters the decision directly and also reduces the relative weight placed on the advisor's recommendation. This creates the strategic incentive for the advisor to adjust its recommendation in anticipation of personal-AI consultation.

\subsection{Optimal Recommendation and Strategic Counteraction}

Before analyzing the advisor's strategic response to personal-AI consultation, consider a benchmark in which the advisor ignores the possibility that personal AI may be consulted and assumes that the decision maker has no prior knowledge of the task, corresponding to a zero-mean uninformative prior. In this benchmark, the advisor believes that its recommendation fully determines the decision maker's final decision. The optimal strategy is therefore simply to align the recommendation with the target:
\[
S_A=r.
\]

We now turn to the case in which the advisor anticipates and strategically responds to possible personal-AI consultation. Under the baseline information structure, the advisor observes or perfectly predicts $S_P$ before choosing $S_A$, but does not know whether personal AI will actually be consulted.

\begin{proposition}[Optimal strategic recommendation]\label{prop:optimal-recommendation}
The advisor's optimal recommendation is
\[
S_A^*
=
r+\frac{1}{r_A}(r-\mu_0)
+\Delta(p,r_A,r_P)(r-S_P),
\]
where the counteraction intensity is
\[
\Delta(p,r_A,r_P)
=
\frac{p r_P(1+r_A)^2}
{r_A\left[(1-p)(1+r_A+r_P)^2+p(1+r_A)^2\right]}
=
\left(1+\frac{1}{r_A}\right)
\frac{pt}{(1-p)(1+t)^2+p}.
\]
The corresponding minimized loss is
\[
L^*
=
\frac{p(1-p)r_P^2}
{(1-p)(1+r_A+r_P)^2+p(1+r_A)^2}
(r-S_P)^2
=
\frac{p(1-p)t^2}
{(1-p)(1+t)^2+p}
(r-S_P)^2.
\]
\end{proposition}

The optimal recommendation consists of three components. The first term aligns the recommendation with the target $r$. The second corrects for the difference between the decision maker's prior mean and the advisor's target. The third is a strategic counteraction term. When $S_P\neq r$, the advisor shifts its recommendation in the opposite direction from the personal-AI recommendation to offset its influence on the final decision. The coefficient $\Delta$ measures the intensity of this counteraction.

We next characterize how counteraction intensity $\Delta$ varies with personal-AI consultation probability $p$, relative trust $t$, and the advisor's perceived precision $r_A$.

\begin{corollary}[Counteraction intensity]\label{cor:counteraction}
The counteraction intensity $\Delta$ satisfies the following properties.
\begin{enumerate}
    \item \textbf{Monotonicity in adoption.} Fix $r_A$ and $r_P>0$. Then $\Delta$ is strictly increasing in $p$. Moreover,
    \[
    \Delta\to 0 \quad \text{as } p\downarrow 0,
    \qquad
    \Delta\to \frac{r_P}{r_A} \quad \text{as } p\uparrow 1.
    \]

    \item \textbf{Non-monotonicity in relative trust.} Fix $p\in(0,1)$ and $r_A$. Then $\Delta$ is hump-shaped in $t$, with unique maximizer
    \[
    t^*=\frac{1}{\sqrt{1-p}}.
    \]
    Moreover,
    \[
    \Delta\to 0 \quad \text{as } t\downarrow 0,
    \qquad
    \Delta\to 0 \quad \text{as } t\uparrow \infty.
    \]

    \item \textbf{Monotonicity in advisor precision, holding relative trust fixed.} Fix $p>0$ and $t>0$. Then $\Delta$ is strictly decreasing in $r_A$.
\end{enumerate}
\end{corollary}

\begin{figure}[htb]
\centering
\includegraphics[scale=0.5]{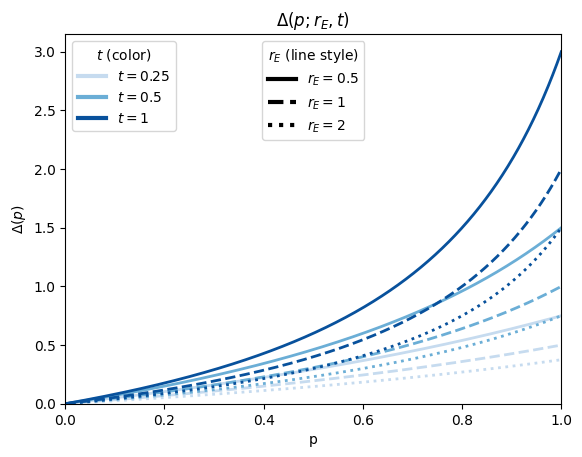}
\hspace{0.03\textwidth}
\includegraphics[scale=0.5]{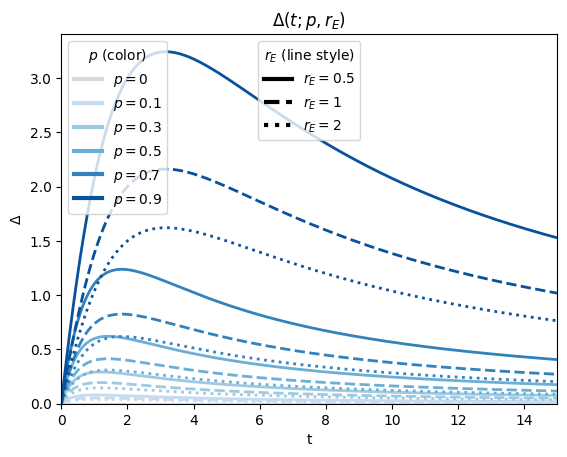}
\caption{Counteraction intensity $\Delta$ increases in adoption rate $p$ (left) and is hump-shaped in relative trust ratio $t$ for fixed $p\in(0,1)$ (right).}
\label{fig:counteraction}
\end{figure}

\paragraph{Monotonicity in adoption.}
When personal AI is rarely consulted, it rarely enters the decision rule, so the advisor has little need or incentive to adjust its recommendation in response to personal AI. As $p$ increases, personal AI is consulted more often, strengthening the advisor's incentive to counteract its influence on the final decision. In the limit as $p\to1$, personal AI is always consulted, and the counteraction intensity converges to $r_P/r_A$.

\paragraph{Non-monotonicity in relative trust.}
When $t$ is small, personal AI has little influence on the final decision, so there is little need or incentive to counteract it. As $t$ increases, personal AI becomes more influential when consulted, strengthening the advisor's incentive to counteract. However, a second force works in the opposite direction: as $t$ increases, counteraction becomes less effective. In the extreme as $t\to+\infty$, personal AI effectively dominates the final decision whenever it is consulted, so any finite counteraction by the advisor becomes ineffective. The advisor therefore optimally chooses no counteraction in the limit. Counteraction intensity is strongest at an intermediate level of relative trust, where personal AI is influential enough to create an incentive to counteract, while the advisor still has sufficient influence for counteraction to be effective. The maximizing level $t^*=1/\sqrt{1-p}$ increases with $p$: when consultation is more likely, the advisor has a stronger incentive to counteract even at higher levels of relative trust.

\paragraph{Advisor precision.}
Holding \(p\) and \(t\) fixed, when personal AI is consulted, the weight placed on personal AI remains unchanged, as does the combined weight placed on the decision maker’s prior and the advisor’s recommendation. Increasing \(r_A\), however, shifts more of this combined weight from the prior toward the advisor. As a result, each unit of adjustment in the advisor's recommendation has a larger effect on the decision maker's final decision. The advisor therefore needs a smaller recommendation adjustment to achieve the same corrective effect. Hence, the counteraction intensity $\Delta$ decreases with $r_A$.

\subsection{Advisor Loss}

We now turn from the advisor's optimal recommendation to the loss it incurs after responding optimally. This loss captures the advisor's remaining vulnerability after optimally adjusting its recommendation. Recall that the minimized loss is
\[
L^*
=
\frac{p(1-p)r_P^2}
{(1-p)(1+r_A+r_P)^2+p(1+r_A)^2}
(r-S_P)^2
=
\frac{p(1-p)t^2}
{(1-p)(1+t)^2+p}
(r-S_P)^2.
\]

Two features of this closed-form expression are worth highlighting. First, $L^*$ does not depend on the decision maker's prior mean $\mu_0$, or equivalently on the disagreement between $\mu_0$ and the advisor's target $r$. Prior misalignment affects the advisor’s optimal recommendation, giving rise to the prior-correction term. This term fully offsets the effect of prior misalignment on the final decision, so prior misalignment does not generate residual loss.

By contrast, the residual loss arises from stochastic personal-AI consultation and depends on the personal-AI recommendation only through the squared deviation $(r-S_P)^2$. The advisor must choose a single recommendation before knowing whether personal AI will be consulted, while the consultation and no-consultation states generally call for different optimal recommendations. A single recommendation therefore cannot implement both simultaneously, creating an irreducible loss under stochastic consultation. In particular, when $S_P=r$, the advisor incurs no loss.

We next characterize how the minimized loss varies with the personal-AI consultation probability $p$ and relative trust $t$, together with the corresponding implications for $r_A$ and $r_P$.

\begin{corollary}[Advisor loss]\label{cor:loss}
The minimized loss $L^*$ satisfies the following properties.
\begin{enumerate}
    \item \textbf{Non-monotonicity in adoption.} Fix $r_A$ and $r_P>0$. Then $L^*=0$ at $p=0$ and $p=1$, and $L^*>0$ for $p\in(0,1)$ whenever $S_P\neq r$. Furthermore, $L^*$ is single-peaked in $p$ and is maximized at
    \[
    p^*
    =
    \frac{1+r_A+r_P}{2(1+r_A)+r_P}
    =
    1-\frac{1}{t+2}.
    \]

    \item \textbf{Relative-trust sufficiency.} Fix $p\in(0,1)$. The minimized loss depends on $r_A$ and $r_P$ only through the relative trust $t$.

    \item \textbf{Monotonicity in relative trust.} Fix $p\in(0,1)$. The minimized loss $L^*$ is strictly increasing in $t$. Specifically,
    \[
    L^*\to0
    \quad\text{as }t\downarrow0,
    \qquad
    L^*\to p(r-S_P)^2
    \quad\text{as }t\uparrow+\infty.
    \]

    \item \textbf{Implications for perceived precision.} Fix $p\in(0,1)$ and $r_P$. Then $L^*$ is strictly decreasing in $r_A$. Fix $p\in(0,1)$ and $r_A$. Then $L^*$ is strictly increasing in $r_P$.
\end{enumerate}
\end{corollary}

\begin{figure}[!htbp]
\centering
\includegraphics[scale=0.5]{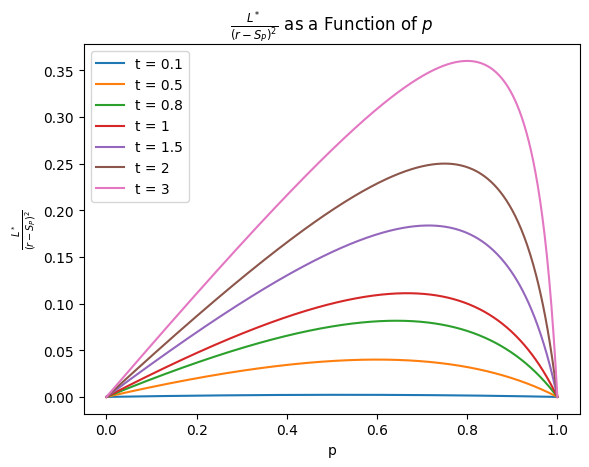}
\hspace{0.03\textwidth}
\includegraphics[scale=0.5]{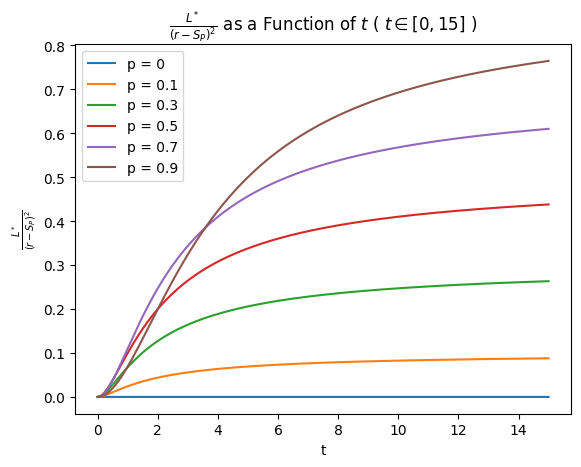}
\caption{Advisor loss $L^*$ is hump-shaped in adoption $p$: it peaks at an interior level of adoption and vanishes at $p=0$ and $p=1$ (left). It is strictly increasing in relative trust $t$ for fixed $p\in(0,1)$ (right).}
\label{fig:loss_p}
\end{figure}

\paragraph{Non-monotonicity in adoption.}
When $p=0$, personal AI is never consulted, so the advisor faces only the no-consultation decision rule and can choose its recommendation to align the final decision with the target $r$. When $p=1$, personal AI is always consulted, so the advisor faces only the consultation decision rule and can fully anticipate and offset the personal-AI recommendation, again choosing its recommendation to align the final decision with the target $r$. In both cases, the advisor incurs no residual loss. When $0<p<1$, however, the advisor does not know whether personal AI will enter the decision rule and must account for both the consultation and no-consultation states when choosing its recommendation. This strategic uncertainty creates unavoidable misalignment and generates a hump-shaped loss, which is maximized at an interior consultation probability.

\paragraph{Relative trust.}
The sufficiency result shows that the advisor's loss depends on $r_A$ and $r_P$ only through the relative trust $t$. When $t=0$, personal AI has no influence on the final decision even when it is consulted, so the advisor can choose its recommendation to align the final decision with the target $r$ regardless of consultation, and the minimized loss is zero. As $t$ increases, personal AI receives greater weight when consulted, widening the gap between the recommendations that would be optimal with and without personal-AI consultation. Because the advisor must use a single recommendation across both states, this widening gap increases the cost of stochastic consultation, and the minimized loss rises with $t$. At the other extreme, as $t\to+\infty$, personal AI effectively determines the final decision whenever it is consulted, and counteraction becomes ineffective. The advisor can still align the no-consultation decision with its target, but bears the full disagreement cost whenever consultation occurs. Consequently, $L^*\to p(r-S_P)^2$.

\paragraph{Implications for perceived precision.}
The implications for $r_A$ and $r_P$ follow directly from their effects on the relative trust ratio $t$.

\subsection{Partially Predictable Personal AI}\label{sec:partial-predictability}

The baseline analysis assumes that the advisor observes or perfectly predicts \(S_P\). We now relax this assumption and allow the advisor to observe only partial information \(H\) about the personal-AI recommendation before choosing \(S_A\). The same strategic logic carries over under partial predictability: the advisor counteracts the predictable component of personal AI, while residual unpredictability generates an additional loss that cannot be eliminated through recommendation design.

Conditional on \(H\), the advisor solves
\[
\min_{S_A}\; E\!\left[(D-r)^2\mid H\right],
\]
taking the decision maker's perceived aggregation rule as given.

\begin{proposition}[Optimal recommendation under partial predictability]\label{prop:partial-predictability}
The optimal recommendation is
\[
S_A^*
=
r+\frac{1}{r_A}(r-\mu_0)
+
\Delta(p,r_A,r_P)
\left(
r-E[S_P\mid H]
\right),
\]
where
\[
\Delta(p,r_A,r_P)
=
\frac{p r_P(1+r_A)^2}
{r_A\left[(1-p)(1+r_A+r_P)^2+p(1+r_A)^2\right]}.
\]
The corresponding minimized conditional loss is
\[
\begin{aligned}
L^*(H)
={}&
\frac{p(1-p)r_P^2}
{(1-p)(1+r_A+r_P)^2+p(1+r_A)^2}
\left(
r-E[S_P\mid H]
\right)^2\\
&+
p\left(\frac{r_P}{1+r_A+r_P}\right)^2
\operatorname{Var}(S_P\mid H).
\end{aligned}
\]
\end{proposition}

The optimal recommendation has the same form as in the fully predictable case, with \(S_P\) replaced by its predictable component \(E[S_P\mid H]\). Thus, the advisor counteracts the component of the personal-AI recommendation that it can predict. The first term in the minimized loss likewise has the same form as the baseline loss after replacing \(S_P\) by \(E[S_P\mid H]\). The second term is new and depends on the conditional variance of \(S_P\), capturing the irreducible loss from the component of the personal-AI recommendation that cannot be predicted when \(S_A\) is chosen. When \(S_P\) is perfectly predictable, this variance is zero and the baseline result is recovered. Thus, the advisor counteracts what it can predict and bears the additional loss associated with what it cannot.

\subsection{Costly Recommendation Adjustment}\label{subsec:costly-rec-adjust}

Section~\ref{sec:partial-predictability} relaxes the assumption that the personal-AI recommendation is perfectly predictable. We now consider a separate relaxation: adjusting the advisor's recommendation may itself be costly. Such costs can capture implementation effort, organizational frictions, or other penalties associated with departing from a benchmark recommendation. For an arbitrary benchmark recommendation $S_0$ and adjustment-cost parameter $\kappa \geq 0$, consider
\begin{equation*}
\min_{S_A}\; E\!\left[(D-r)^2\right] + \kappa(S_A-S_0)^2.
\end{equation*}

To characterize this problem, define
\[
W
=
(1-p)\left(\frac{r_A}{1+r_A}\right)^2
+
p\left(\frac{r_A}{1+r_A+r_P}\right)^2
=
\left(\frac{r_A}{1+r_A}\right)^2
\left[(1-p)+\frac{p}{(1+t)^2}\right].
\]
The coefficient $W$ measures the advisor's effective leverage over the final decision: it is the probability-weighted squared influence of the advisor's recommendation across the consultation and no-consultation states. Equivalently, $W$ determines the decision-loss cost of moving away from the costless optimum $S_A^*$. The advisor's objective under costly recommendation adjustment can therefore be written as
\[
E\!\left[(D-r)^2\right] + \kappa(S_A-S_0)^2
=
L^*+W(S_A-S_A^*)^2+\kappa(S_A-S_0)^2,
\]
where $S_A^*$ is the costless optimal recommendation and $L^*$ is the corresponding minimized loss under costless counteraction, both characterized in Proposition~\ref{prop:optimal-recommendation}. Here, $W(S_A-S_A^*)^2$ is the additional decision loss from moving away from the costless optimum $S_A^*$, whereas $\kappa(S_A-S_0)^2$ is the adjustment cost from moving away from the benchmark $S_0$.

\begin{proposition}[Optimal recommendation under costly adjustment]\label{prop:costly-adjustment}
For any benchmark recommendation $S_0$ and $\kappa \geq 0$, the optimal recommendation is
\[
S_{A,\kappa}^*
=
S_A^*+\lambda_\kappa(S_0-S_A^*),
\]
where the shrinkage coefficient $\lambda_\kappa=\kappa/(W+\kappa)$. The corresponding minimized total loss is
\[
L_\kappa^*
=
L^*+\alpha_\kappa(S_A^*-S_0)^2,
\]
where the additional-loss coefficient $\alpha_\kappa=\kappa W/(W+\kappa)$.
\end{proposition}

The optimal recommendation is therefore a weighted average of the costless optimum and the benchmark:
\[
S_{A,\kappa}^*
=
(1-\lambda_\kappa)S_A^*+\lambda_\kappa S_0.
\]
The shrinkage coefficient $\lambda_\kappa$ measures how far the costly optimal recommendation moves from $S_A^*$ toward $S_0$. It is strictly increasing in $\kappa$: as $\kappa\downarrow0$, $\lambda_\kappa\to0$ and the costless optimum is recovered, whereas as $\kappa\to\infty$, $\lambda_\kappa\to1$ and the optimal recommendation converges to the benchmark.

The minimized total loss also has a simple decomposition. The first term, $L^*$, is the loss attainable under costless adjustment. The second term captures the additional loss generated by the tension between the costless optimum and the benchmark. For $W,\kappa>0$,
\[
0<\alpha_\kappa<\min\{W,\kappa\}.
\]
Moreover, $\alpha_\kappa$ is strictly increasing in $\kappa$: as $\kappa\downarrow0$, $\alpha_\kappa\to0$, and as $\kappa\to\infty$, $\alpha_\kappa\to W$. Thus, the costly-adjustment problem balances the decision-loss cost of departing from $S_A^*$ against the adjustment cost of departing from $S_0$.

The coefficients $\lambda_\kappa$ and $\alpha_\kappa$ also vary systematically with the underlying model parameters through $W$. The following corollary summarizes these comparative statics.

\begin{corollary}[Comparative statics of costly adjustment]\label{cor:costly-adjustment}
Fix $\kappa>0$.
\begin{enumerate}
    \item The shrinkage coefficient $\lambda_\kappa$ is strictly decreasing in $r_A$, strictly increasing in $p$ when $t>0$, and strictly increasing in $t$ when $p>0$.
    \item The additional-loss coefficient $\alpha_\kappa$ is strictly increasing in $r_A$, strictly decreasing in $p$ when $t>0$, and strictly decreasing in $t$ when $p>0$.
\end{enumerate}
\end{corollary}

These comparative statics follow from how $r_A$, $p$, and $t$ affect the advisor's effective leverage $W$. A larger $r_A$ increases the advisor's effective leverage and hence increases $W$, so moving away from $S_A^*$ has a larger effect on the final decision and generates a larger decision loss. The adjustment cost is then relatively less important, so $\lambda_\kappa$ falls and the costly optimal recommendation remains closer to $S_A^*$. At the same time, because deviations from $S_A^*$ are more consequential, the additional-loss coefficient $\alpha_\kappa$ rises. By contrast, a larger $p$ or $t$ reduces the advisor's effective leverage and hence $W$. The adjustment cost then becomes relatively more important, so $\lambda_\kappa$ rises and $\alpha_\kappa$ falls. These comparative statics concern the shrinkage coefficient and the additional-loss coefficient rather than the overall recommendation or total loss, because $S_A^*$, $L^*$, and the distance $S_A^*-S_0$ may themselves depend on the underlying parameters.

When recommendation adjustment is prohibitively costly,
the advisor may instead invest in perceived credibility.
Appendix~\ref{app:trust-investment} examines this alternative
response by fixing $S_A=r$, corresponding to the limiting case
$S_0=r$ and $\kappa\to\infty$, and allowing the advisor to
increase perceived trust at a cost.

\section{Perceived Private Information}

Our baseline model is sufficiently general to accommodate a range of information structures. In this section, we consider one important example in which the decision maker believes that personal AI has access to user-specific information or context that is unavailable to the focal advisor. Such information may include personal history, prior interactions, preferences, or other user-specific context that personal AI can access or utilize through personalization.

We decompose the unknown state as
\[
\theta = X + Y,
\]
where $X$ captures the component associated with common or task-level information and $Y$ captures the component associated with user-specific information. Let $\eta \in [0,1)$ denote the share of task uncertainty attributable to the user-specific component. When $\eta=0$, the model reduces to the baseline setting. As $\eta$ approaches one, essentially all task uncertainty is attributed to the user-specific component. We show that this richer information structure preserves the form of the baseline model, with the main results continuing to hold after replacing the baseline perceived precisions with their effective counterparts under private information.

\subsection{Private-Information Structure}

The decision maker holds a prior belief
\[
\theta = X+Y, \qquad X \perp Y,
\]
where
\[
X \sim \mathcal{N}(\mu_0,\sigma_X^2), \qquad
Y \sim \mathcal{N}(0,\sigma_Y^2), \qquad
\theta \sim \mathcal{N}(\mu_0,\sigma_0^2), \qquad
\sigma_0^2 = \sigma_X^2 + \sigma_Y^2.
\]

The decision maker believes that the advisor's recommendation is a noisy signal of the common, task-level component,
\[
S_A = X + \epsilon_A, \qquad \epsilon_A \sim \mathcal{N}(0,\sigma_A^2),
\]
whereas the personal AI recommendation is a noisy signal of the full state,
\[
S_P = \theta + \epsilon_P, \qquad \epsilon_P \sim \mathcal{N}(0,\sigma_P^2).
\]
Thus, personal AI is perceived as being able to combine task-level information with user-specific information that is unavailable to the focal advisor.

We define
\[
\eta
= \frac{\sigma_Y^2}{\sigma_X^2+\sigma_Y^2}
= \frac{\sigma_Y^2}{\sigma_0^2}
\]
as the share of task uncertainty attributable to the user-specific component. When $\eta=0$, $Y$ is degenerate and the information structure reduces to the baseline model. A larger $\eta$ means that a greater share of uncertainty is associated with information that the advisor does not observe but that personal AI is perceived as being able to utilize.

The zero mean of $Y$ is without loss of generality. If the user-specific component instead has a nonzero mean $\mu_Y$, the decision maker's aggregation rules depend on the advisor recommendation through $S_A+\mu_Y$. The optimal recommendation $S_A^*$ therefore shifts by $-\mu_Y$, while the effective perceived precisions, relative trust ratio, counteraction intensity, and minimized loss remain unchanged. We therefore normalize $Y$ to have zero mean.

One interpretation of the user-specific component $Y$ is that the decision maker has a set of personal or contextual attributes $Z_1,\ldots,Z_n$, but does not know how these attributes map into the decision task. Let $f(Z_1,\ldots,Z_n)$ denote their task-relevant contribution. Personal AI may use prior interactions, personalization, or other user-specific data to infer or utilize this mapping, whereas the focal advisor does not have access to the same information. The personal AI recommendation $S_P$ can therefore be interpreted as an estimate of the full state that incorporates this user-specific contribution, while $\epsilon_P$ captures imperfections in that estimate. Under our normalization, $Y$ represents this task-relevant contribution relative to its prior mean.

\subsection{Effective Trust and Advisor Vulnerability}

We now characterize how the private-information structure changes the effective perceived precisions that enter the baseline model. The key observation is that the structure of the model is unchanged: private information changes the effective values of $r_A$ and $r_P$, but not the form of the advisor's problem.

\begin{proposition}[Effective perceived precisions under private information]\label{prop:private-information-precision}
For $\eta \in [0,1)$, the conclusions of Proposition~\ref{prop:optimal-recommendation} and Corollaries~\ref{cor:counteraction} and~\ref{cor:loss} continue to hold with $r_A$ and $r_P$ replaced by
\[
r_{A,\eta}=(1-\eta)r_A,
\qquad
r_{P,\eta}=\bigl[1+\eta(1-\eta)r_A\bigr]r_P.
\]
The corresponding relative trust ratio is
\[
t_\eta
=\frac{r_{P,\eta}}{1+r_{A,\eta}}
=\frac{\bigl[1+\eta(1-\eta)r_A\bigr]r_P}
{1+(1-\eta)r_A}.
\]
\end{proposition}

Thus, the private-information extension does not require a new solution to the advisor's problem; it changes only the effective perceived precisions that enter the baseline formulas.

\begin{lemma}[Private information and relative trust]\label{lem:private-information-trust}
For $r_A>0$ and $r_P>0$, $r_{A,\eta}$ is strictly decreasing in $\eta$, while $t_\eta$ is strictly increasing in $\eta$ on $[0,1)$.
\end{lemma}

The effect of $\eta$ on counteraction intensity is more subtle. A larger $\eta$ lowers $r_{A,\eta}$, so a larger recommendation adjustment is required to exert a given influence on the final decision. At the same time, it raises $t_\eta$. Recall from Corollary~\ref{cor:counteraction} that, holding advisor precision fixed, counteraction intensity is increasing in relative trust up to $t^*=1/\sqrt{1-p}$ and decreasing thereafter. Hence, when $t_\eta\leq t^*$, both forces increase counteraction; once $t_\eta>t^*$, they operate in opposite directions, and the overall effect of $\eta$ on counteraction intensity can be non-monotone.

\begin{corollary}[Private information and advisor vulnerability]\label{cor:private-information-loss}
Fix $p\in(0,1)$ and $S_P\neq r$. The advisor's minimized loss $L^*$ is strictly increasing in $\eta$ on $[0,1)$.
\end{corollary}

A larger $\eta$ increases the relative influence of personal AI through $t_\eta$. Because the minimized loss is strictly increasing in relative trust, this widens the gap between the consultation and no-consultation environments and increases advisor vulnerability.

\section{Experimental Setting and Design}
\label{sec:experiment}

We designed the City Manager Game, a controlled online forecasting experiment in which participants act as city managers. In each round, a participant uses information on a task card to form an initial forecast, has an opportunity to receive advice, and submits a final forecast. Advice is presented as coming from enterprise AI, personal AI, or a human expert. Enterprise AI represents the focal advisor in the model, personal AI represents the outside source, and the human expert provides a non-AI comparison.

The experiment consists of three blocks, each addressing a different part of the advisor's problem. Block 1 measures how participants combine their initial forecasts with advice from one or two sources. Because the recommendations appear automatically, forecast revisions describe reliance once advice is available. Block 2 examines whether participants obtain personal-AI advice when access requires entering a code within a time limit. Block 3 compares unadjusted enterprise advice with naïve counteracting policy. Together, the blocks distinguish reliance on displayed advice, consultation of an optional source, and performance under a fixed counteraction policy.

\subsection{Forecasting Task and Advice Sources}
\label{subsec:task-environment}
\label{subsec:controlled-advice}

Each round contains one forecasting task concerning finance and commerce, public health, or parks and recreation. A task card describes the municipal problem, provides information for forming a forecast, and specifies a feasible forecast range. For example, one task asks participants to forecast the number of applications to a municipal grant program using information about earlier programs. The task has a predetermined outcome against which forecasts are evaluated.

Numerical recommendations are specified in advance for each task and presented through controlled advice interfaces. Enterprise AI is described as a municipal decision-support system, personal AI as an external assistant, and the human expert as a professional advisor. These source descriptions accompany assigned recommendations, rather than forecasts obtained from live AI systems or experts. Specifying recommendations in advance allows us to control their content and the requirements for obtaining them.

\subsection{Experimental Flow}
\label{subsec:experimental-flow}

Participants complete four practice rounds followed by 32 main rounds: eight in Block 1, twelve in Block 2, and twelve in Block 3. Practice introduces the task without advice, then with enterprise AI, and finally with both enterprise and personal AI. Table~\ref{tab:experiment-flow} summarizes the main rounds, numbered independently of practice.

\begin{table}[!htb]
\small\centering
\caption{Experimental Flow}
\label{tab:experiment-flow}
\begin{tabular}{p{0.10\textwidth}p{0.14\textwidth}p{0.62\textwidth}}
\hline
Block & Main rounds & Advice environment \\
\hline
1 & R1--R2 & Enterprise AI only \\
 & R3--R4 & Personal AI only \\
 & R5--R6 & Human expert only \\
 & R7--R8 & Enterprise and personal AI \\
2 & R9--R10 & Both AI recommendations displayed automatically \\
 & R11--R12 & Optional personal AI: 4-character code, 10 seconds \\
 & R13--R14 & Optional personal AI: 6-character code, 10 seconds \\
 & R15--R16 & Optional personal AI: 8-character code, 10 seconds \\
 & R17--R18 & Optional personal AI: 10-character code, 10 seconds \\
 & R19--R20 & Optional personal AI: 10-character code, 8 seconds \\
3 & R21--R32 & Enterprise AI throughout; personal-AI availability and counteracting advice vary within groups of four rounds \\
\hline
\end{tabular}
\begin{minipage}{0.92\textwidth}\vspace{0.5em}\footnotesize
Notes. All participants encounter blocks and scenarios in the same order. The private-context cue is randomized within each pair of Block 2 rounds. Block 3 assignments are shuffled within each group of four rounds. Four practice rounds precede R1.
\end{minipage}
\end{table}

Block 1 provides two rounds with each of four advice configurations: enterprise AI only, personal AI only, a human expert only, and both AI sources. The recommendations in each configuration appear automatically after the initial forecast. Comparing initial and final forecasts allows us to estimate reliance on the displayed advice.

Block 2 keeps enterprise advice visible and varies the requirements for obtaining personal-AI advice. Both recommendations appear automatically in the first two rounds. In the remaining ten rounds, participants can choose to enter a code before a time limit expires or proceed without attempting access. Correct entry within the time limit reveals the personal-AI recommendation; an unsuccessful attempt leaves it hidden. We refer to successful access as consultation. This design distinguishes a decision to seek advice from actually obtaining the recommendation.

The ten rounds comprise five consecutive pairs: four-, six-, eight-, and ten-character codes with a ten-second limit, followed by ten-character codes with an eight-second limit. Within each pair, one round includes a private-context cue and the other does not, with randomized order. The cue describes personal AI as drawing on the participant's previous rounds without changing its assigned numerical recommendation.

Block 3 varies personal-AI availability and enterprise recommendation policy over three groups of four rounds. Personal AI is available in one, two, and three rounds, respectively. Counteracting enterprise advice appears in one, two, and two rounds, always when personal AI is available. These assignments are shuffled within each group, and available advice appears automatically. Private-context framing is present throughout the block. An enterprise certification cue is assigned at the participant level and held fixed within the block.

Counteraction places the enterprise recommendation on the opposite side of the task outcome from personal AI, at the same distance. For round $t$, let $y_t$ denote the task outcome and let $S_{E,t}$ and $S_{P,t}$ denote the enterprise and personal-AI recommendations. The enterprise recommendation corresponds to $S_A$ in the theory. The counteracting recommendation is
\begin{equation}
S_{E,t}^{\mathrm{c}}=2y_t-S_{P,t}=y_t+(y_t-S_{P,t}),
\label{eq:implemented-counteraction}
\end{equation}
subject to the feasible range $[0,1000]$. It is constructed from the task outcome and personal-AI recommendation, without participant-specific weight calibration. Participants see the resulting recommendation without being told that a counteraction policy is used. When the range constraint does not bind, an equal-weight average of the two recommendations equals the outcome exactly.

\subsection{Procedure and Measures}
\label{subsec:procedure-incentives}
\label{subsec:survey-measures}

The experiment was administered online as part of a university course assignment. Participants completed the task individually and were instructed to make accurate forecasts. Each round elicited an initial forecast and confidence rating, provided the advice stage, and elicited a final forecast and confidence rating. The game score assigned 25\% of the round score to the initial forecast and 75\% to the final forecast; each component decreased with absolute error, subject to a floor of zero. Practice did not contribute to the total score. Outcome feedback followed each round in Blocks 1 and 2 and each group of four rounds in Block 3.

The platform records forecasts, numerical recommendations, source availability, access attempts and outcomes, confidence, timing, and treatment assignments. A post-game survey elicits stated perceptions of the advice sources and reported reliance. The analyses below use forecast and access behavior.

We measure forecast accuracy using absolute error, the distance between the forecast and the task outcome. For a given participant in round $t$, let $x_{0,t}$ and $x_{1,t}$ denote the initial and final forecasts and let $y_t$ denote the task outcome. Forecast improvement is $|x_{0,t}-y_t|-|x_{1,t}-y_t|$, so a positive value indicates that the final forecast is more accurate. The initial forecast provides a comparison within the same task, before advice is displayed.

For the ten code-required rounds, we also measure how often participants initiate the code-entry process and how often they successfully obtain the recommendation. We call these events an access attempt and consultation, respectively. Reporting both measures allows us to distinguish rounds without an attempt from rounds in which an attempt fails to reveal advice.

To measure reliance, we fit the final forecast as a weighted average of the initial forecast and the displayed recommendations. The weights are nonnegative, sum to one, and minimize squared forecast residuals. We pool observations across participants within each advice configuration; a larger weight on a recommendation indicates greater reliance under the fitted rule. Appendix~\ref{subsec:measurement-estimation} gives the estimators, participant-level resampling procedure, and alternative weight-of-advice summaries.

\subsection{Empirical Comparisons}
\label{subsec:design-scope}

The three blocks support different empirical comparisons. Blocks 1 and 2 describe forecast revisions and consultation across the advice environments participants encounter. Their fixed sequences do not separate task and order differences from source identity or access requirements. In Block 3, the third group contains both counteracting and unadjusted enterprise advice with personal AI available. The shuffled assignments within this group permit a within-participant policy comparison while holding personal-AI availability fixed.

The fitted aggregation weights summarize forecast revisions under a linear rule. Interpreting them as perceived precision requires the model's aggregation assumptions. Similarly, initial-to-final accuracy changes compare two judgments within the same task, without a concurrent no-advice condition. We use these comparisons to describe advice use and the performance of the implemented recommendation policy.

\section{Experimental Evidence}
\label{sec:results}

\subsection{Sample and Forecast Accuracy}
\label{subsec:results-sample}

The analysis uses a common sample of 69 participants, each completing four practice rounds and 32 main rounds, yielding 2,208 main-task observations. Practice is excluded throughout.

Final forecasts are more accurate on average in every block. Mean absolute error falls from 69.4 to 45.6 in Block 1, from 55.8 to 35.3 in Block 2, and from 69.3 to 41.1 in Block 3. These changes describe improvements over the advice stage within each set of tasks.

\subsection{Reliance on Advice}
\label{subsec:results-block1}

We first examine how final forecasts combine initial judgments with displayed advice. In the single-source rounds of Block 1, the estimated weights on the recommendation are 0.65 for enterprise AI, 0.55 for personal AI, and 0.55 for the human expert (Table~\ref{tab:single-source-weights}). Each estimate describes how the final forecast combines the initial judgment with one source of advice. The table also reports the uncertainty around these pooled estimates; the comparisons concern the source configurations on the tasks used in the experiment.

\begin{table}[!htb]
\small\centering
\caption{Reliance in Single-Source Rounds}
\label{tab:single-source-weights}
\begin{tabular}{lcc}
\hline
Advice source & Pooled weight & 95\% confidence interval \\
\hline
Enterprise AI & 0.65 & [0.42, 0.79] \\
Personal AI & 0.55 & [0.36, 0.70] \\
Human expert & 0.55 & [0.37, 0.71] \\
\hline
\end{tabular}
\begin{minipage}{0.92\textwidth}\vspace{0.5em}\footnotesize
Notes. Each source condition contains 138 rounds from 69 participants. Weights minimize squared forecast-revision residuals subject to lying in $[0,1]$. Intervals are percentile intervals from 20,000 participant bootstrap samples. Appendix~\ref{subsec:single-source-robustness} reports weight-of-advice definitions and sensitivity.
\end{minipage}
\end{table}

With both AI recommendations displayed, the final forecast combines three inputs. The fitted weights are 0.35 on the initial forecast, 0.37 on enterprise AI, and 0.28 on personal AI (Figure~\ref{fig:block1-weights}). Both recommendations contribute to the fitted aggregate, while roughly one-third of the weight remains on the participant's initial judgment. Thus, the dual-source results do not reduce to a comparison of which AI receives more weight; the initial forecast remains a substantial component of the fitted decision rule.

To relate this aggregation rule to the theory, let $\widehat w_0$, $\widehat w_E$, and $\widehat w_P$ denote the estimated weights on the initial forecast, enterprise AI, and personal AI, respectively. The relative-influence index is $\widehat t=\widehat w_P/(\widehat w_0+\widehat w_E)=0.39$, with a 95\% bootstrap interval of $[0.22,0.60]$. This index compares personal AI's fitted weight with the combined weight on enterprise advice and the initial judgment. It concerns the influence of personal AI when both recommendations are displayed, rather than how often participants would choose to consult it.

\begin{figure}[!htb]
\centering
\includegraphics[width=0.70\textwidth]{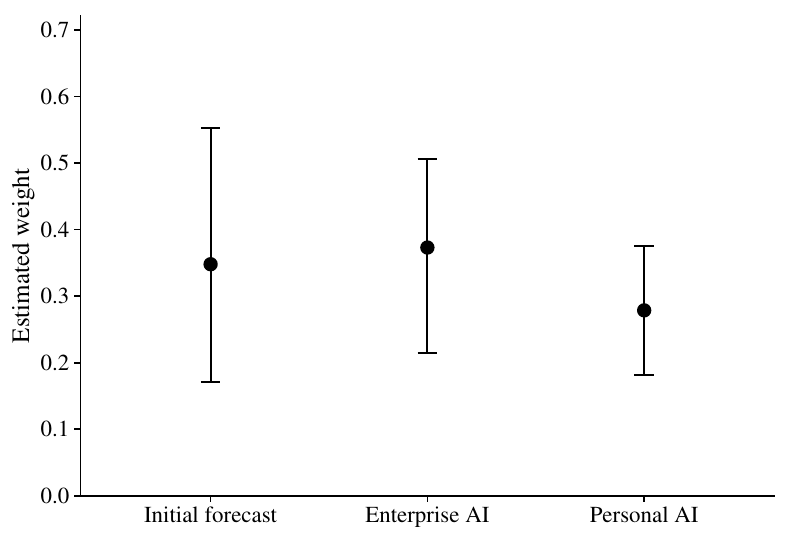}
\caption{Aggregation weights in dual-source rounds. Points are pooled constrained least-squares estimates from 138 rounds completed by 69 participants. Vertical intervals are 95\% percentile intervals from 20,000 participant bootstrap samples. The three weights are nonnegative and sum to one.}
\label{fig:block1-weights}
\end{figure}

For the single-source rounds, an alternative measure, weight of advice, records the fraction of the distance from the initial forecast to the recommendation covered by the forecast revision. Averaging this measure within participants and then across participants gives 0.38 for enterprise AI, 0.41 for personal AI, and 0.44 for the human expert, after excluding advice gaps below five units and truncating ratios to $[0,1]$. Unlike this participant-averaged measure, pooled least squares gives greater influence to rounds with larger initial-to-advice gaps. The precise definitions and sensitivity to the gap threshold and truncation rule appear in Appendix~\ref{subsec:single-source-robustness}.

\subsection{Access to Personal AI}
\label{subsec:results-friction}

Block 2 examines whether participants obtain the personal-AI recommendation when access requires an additional action. In its ten code-required rounds, obtaining advice requires both an access attempt and successful code entry. Figure~\ref{fig:access-outcomes}(a) reports the proportion of rounds with each outcome.

Consultation occurs in 50.7\% of rounds with four-character codes, 43.5\% with six characters, 39.9\% with eight characters, and 24.6\% when the two ten-character conditions are pooled. The corresponding attempt rate declines more gradually, from 53.6\% in the four-character condition to 42.4\% in the pooled ten-character conditions. The figure separates the ten- and eight-second limits used with ten-character codes. The difference between the two series reflects attempts that do not reveal advice.

\begin{figure}[!htb]
\centering
\begin{minipage}[t]{0.49\textwidth}
\centering
\includegraphics[width=\linewidth]{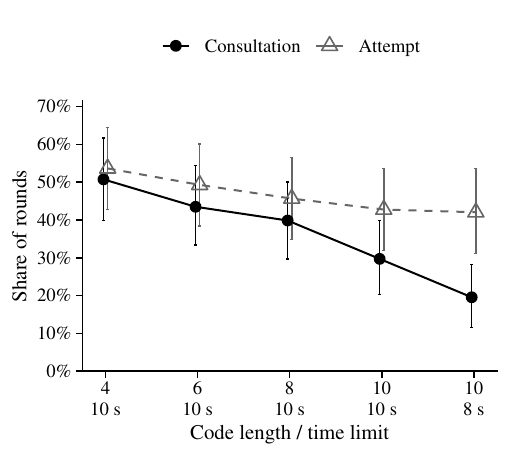}
\small (a) Attempts and consultation
\end{minipage}\hfill
\begin{minipage}[t]{0.49\textwidth}
\centering
\includegraphics[width=\linewidth]{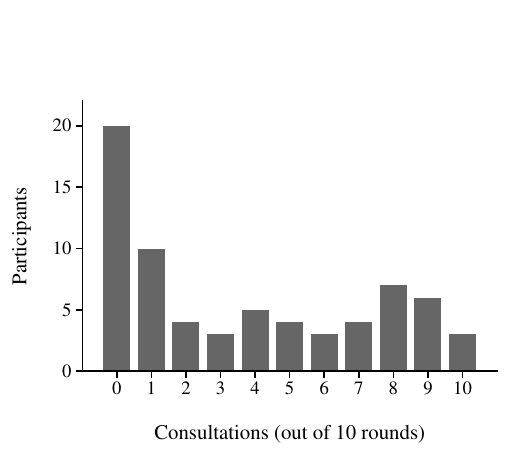}
\small (b) Consultation across participants
\end{minipage}
\caption{Access to personal AI. Panel (a) reports attempt and consultation rates in the implemented code-length and time-limit conditions. Each condition contains two rounds from each of 69 participants. Intervals are 95\% percentile intervals from 20,000 participant bootstrap samples. Lines connect the ordered conditions for readability; code length, scenario, and position vary together. Panel (b) counts participants by consultations across all ten code-required rounds. Consultation means that advice was revealed. Zero-code rounds display advice automatically and are excluded from both panels.}
\label{fig:access-outcomes}
\end{figure}

In the ten-character conditions, many attempts do not result in consultation. Each condition contains 138 rounds from 69 participants. With a ten-second limit, 59 attempts yield 41 consultations; with an eight-second limit, 58 attempts yield 27 consultations. Attempt rates are therefore similar, at 42.8\% and 42.0\%, while consultation rates are 29.7\% and 19.6\%, respectively. Under the eight-second limit, 31 of 58 attempts fail to reveal advice. Thus, nonconsultation includes unsuccessful efforts to obtain a recommendation as well as rounds without an attempt.

The aggregate consultation rates also conceal substantial differences across participants. Of the 69 participants, 20 never consult personal AI in the ten code-required rounds, while three of the 69 consult it in all ten (Figure~\ref{fig:access-outcomes}(b)). These counts motivate a comparison with a benchmark in which all participants consult with a common probability, independently across rounds.

Consider the two rounds with four-character codes. The aggregate consultation rate would imply approximately 17 participants with no consultations, 34 with one, and 18 with two under this benchmark. The observed counts are 28, 12, and 29, respectively. Consultation is more concentrated at zero and two than the benchmark predicts. The observed counts depart from the common-probability, independent-round benchmark at every code length (simulation-based $p<0.001$ throughout). The tests concern these two assumptions jointly; Appendix~\ref{subsec:consultation-benchmark} reports the counts and tests by code length.

The distinction from Block 1 is that personal AI can be offered without its recommendation being consulted. The aggregation weights describe how displayed advice enters a forecast, whereas the consultation rate describes how often participants obtain that advice. These are separate inputs to the focal advisor's recommendation problem.

\subsection{Counteracting Recommendations}
\label{subsec:results-block3}

The final comparison asks how participants perform when enterprise advice offsets personal AI's error. We use the third group of Block 3 to compare policies while holding personal-AI availability fixed. This is the only group containing both counteracting and unadjusted enterprise advice with personal AI available. Each participant contributes two counteracting rounds and one unadjusted round, yielding 138 and 69 observations, respectively. Across the full block, all 345 counteracting recommendations satisfy $S_{E,t}^{\mathrm{c}}=2y_t-S_{P,t}$ exactly, so the range constraint never alters the intended policy.

\begin{figure}[!htb]
\centering
\includegraphics[width=0.78\textwidth]{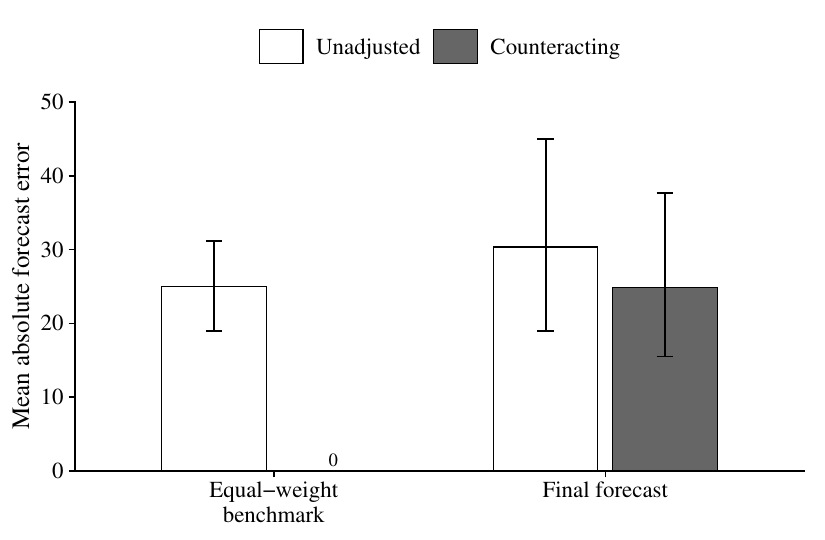}
\caption{Counteraction and forecast accuracy. Bars report mean absolute error in the third group of Block 3, restricted to rounds with personal AI available: two counteracting rounds and one unadjusted round per participant. Means average within participant and then across 69 participants; intervals are 95\% participant bootstrap intervals. The equal-weight benchmark averages the two displayed recommendations and has zero error under counteraction by construction. The final forecast is the participant's observed decision. The paired difference in final error is $-5.5$, with a 95\% paired $t$ interval of $[-15.4,4.5]$.}
\label{fig:counteraction-performance}
\end{figure}

Mean final absolute error is 24.9 under counteraction and 30.3 under unadjusted advice. Averaging within participant before comparing conditions gives a difference of $-5.5$, with a 95\% paired interval of $[-15.4,4.5]$. The estimate favors counteraction but is imprecise. A specification with participant and scenario fixed effects gives a difference of $-7.3$ (standard error 5.8); Appendix~\ref{subsec:counteraction-analysis} reports the specifications, condition means (Figure~\ref{fig:block3-counteraction}), and participant-level distribution.

The equal-weight benchmark in Figure~\ref{fig:counteraction-performance} distinguishes the numerical construction of counteraction from the participant's response to it. Averaging the two displayed recommendations produces zero error under counteraction, compared with 25.0 under unadjusted advice. That construction alone does not imply an accurate final forecast when the participant also uses an initial judgment or assigns different weights to the two recommendations.

Let $D_t=w_0x_{0,t}+w_ES_{E,t}+w_PS_{P,t}$ denote the forecast implied by a linear aggregation rule, with weights on the initial judgment, enterprise AI, and personal AI summing to one. Substituting the counteracting recommendation gives
\[
D_t-y_t=w_0(x_{0,t}-y_t)+(w_P-w_E)(S_{P,t}-y_t).
\]
The first term is the contribution of the initial forecast's error. It can remain even when the two recommendations receive equal weight. The second term captures the contribution of unequal advice weights, which can prevent the recommendations' errors from canceling. These terms can reinforce or offset one another, depending on the initial forecast and the relative weights.

This decomposition connects the implemented policy to Proposition~\ref{prop:optimal-recommendation}: recommendation design must account for the decision maker's prior as well as the competing advice. The Block 1 estimates illustrate an aggregation rule in which both contributions can arise. The Block 3 comparison measures the performance of the fixed policy actually implemented. Counteraction must therefore be evaluated at the final-decision stage, where the initial judgment and advice weights determine whether the numerical adjustment achieves its purpose.

\SingleSpacedXI

\ifAIRecArxivStyle
\bibliographystyle{informs2014}
\bibliography{references}
\else
\AIRecPrintMainBibliography
\fi

\AIRecEndMainBibunit

\newpage


\ECSwitch

\vspace{-0.7cm}

\ECHead{Appendix}

\begin{APPENDICES}

\ifAIRecArxivStyle
\else
\AIRecBeginAppendixBibunit

\SingleSpacedXI

\renewcommand{\refname}{References for the E-Companion}
\AIRecPrintAppendixBibliography
\AIRecEndAppendixBibunit
\fi

\normalsize
\section{Proofs}
\label{app-proofs}

This appendix provides proofs of all theoretical results in the main text.

\subsection{Proof of Proposition~\ref{prop:optimal-recommendation}}

\begin{proof}
Fix the realized $S_P$. The advisor chooses $S_A$ to minimize 
\begin{align*}
L(S_A) = E[(D - r)^2 | S_P] = 
(1-p)\left(r - \frac{\mu_0 + r_A S_A}{1 + r_A}\right)^2 + p\left(r - \frac{\mu_0 + r_A S_A + r_P S_P}{1 + r_A + r_P}\right)^2.
\end{align*}
We rewrite the expected loss as a linear combination of $r$, $r - \mu_0$, $r - S_A$, and $r - S_P$:
\begin{align*}
L = &\ (1 - p)\left[
\frac{1}{1 + r_A}(r - \mu_0)
+ \frac{1}{1 + r_A} r_A (r - S_A)
\right]^2\\
+ &\ p\left[
\frac{1}{1 + r_A + r_P}(r - \mu_0)
+ \frac{1}{1 + r_A + r_P} r_A (r - S_A)
+ \frac{1}{1 + r_A + r_P} r_P (r - S_P)
\right]^2.
\end{align*}
This is a quadratic function of $x = r - S_A$ with positive leading coefficient, hence strictly convex, so any solution to the first-order condition is the unique global minimizer. 

Differentiating $L$ with respect to $x$ gives
\[
\begin{aligned}
\frac{\partial L}{\partial x}
={}&\frac{2(1-p)r_A}{1+r_A}
\left[
\frac{r-\mu_0}{1+r_A}
+\frac{r_Ax}{1+r_A}
\right]\\
&+\frac{2pr_A}{1+r_A+r_P}
\left[
\frac{r-\mu_0}{1+r_A+r_P}
+\frac{r_Ax}{1+r_A+r_P}
+\frac{r_P(r-S_P)}{1+r_A+r_P}
\right].
\end{aligned}
\]
Setting the derivative to zero and dividing by $2r_A$ gives
\[
\left[
\frac{1-p}{(1+r_A)^2}
+\frac{p}{(1+r_A+r_P)^2}
\right]\bigl[(r-\mu_0)+r_Ax\bigr]
+
\frac{pr_P}{(1+r_A+r_P)^2}(r-S_P)
=0.
\]
Solving for $x$ yields
\begin{align*}
x^*
= - \frac{1}{r_A}(r-\mu_0)
- \frac{r_P p (1+r_A)^2}
{r_A\big[(1-p)(1+r_A+r_P)^2+p(1+r_A)^2\big]}(r-S_P).
\end{align*}
Substituting back gives
\[
S_A^*
=
r+\frac{1}{r_A}(r-\mu_0)
+\Delta(p,r_A,r_P)(r-S_P),
\]
where
\[
\Delta(p,r_A,r_P)=
\frac{p r_P(1+r_A)^2}
{r_A\bigl[(1-p)(1+r_A+r_P)^2+p(1+r_A)^2\bigr]}.
\]
Using $r_P=t(1+r_A)$, we obtain
\[
\begin{aligned}
\Delta(p,r_A,r_P)
&=
\frac{pt(1+r_A)^3}
{r_A(1+r_A)^2\bigl[(1-p)(1+t)^2+p\bigr]}\\
&=
\left(1+\frac{1}{r_A}\right)
\frac{pt}{(1-p)(1+t)^2+p}.
\end{aligned}
\]
Substituting $S_A^*$ into the loss function gives
\[
\begin{aligned}
L^*
&=
\left[
(1-p)\left(\frac{r_A\Delta}{1+r_A}\right)^2
+p\left(\frac{r_A\Delta-r_P}{1+r_A+r_P}\right)^2
\right](r-S_P)^2\\
&=
\frac{
(1-p)p^2r_P^2(1+r_A)^2
+p(1-p)^2r_P^2(1+r_A+r_P)^2
}{
\bigl[(1-p)(1+r_A+r_P)^2+p(1+r_A)^2\bigr]^2
}
(r-S_P)^2\\
&=
\frac{p(1-p)r_P^2}
{(1-p)(1+r_A+r_P)^2+p(1+r_A)^2}
(r-S_P)^2.
\end{aligned}
\]
Using $r_P=t(1+r_A)$, we obtain
\[
\begin{aligned}
L^*
&=
\frac{p(1-p)t^2(1+r_A)^2}
{(1+r_A)^2\bigl[(1-p)(1+t)^2+p\bigr]}
(r-S_P)^2\\
&=
\frac{p(1-p)t^2}
{(1-p)(1+t)^2+p}
(r-S_P)^2.
\end{aligned}
\]
\end{proof}

\subsection{Proof of Corollary~\ref{cor:counteraction}}

\begin{proof}
\begin{enumerate}
\item \textbf{Monotonicity in adoption.} Recall that 
\begin{equation*}
{
\Delta(p,r_A,r_P)
=\frac{p\,r_P(1+r_A)^2}
{r_A\big[(1-p)(1+r_A+r_P)^2+p(1+r_A)^2\big]}.
}
\end{equation*}
We can rewrite it as 
\begin{equation}
{
\Delta(p,r_A,r_P)
=
\frac{r_P}{r_A}
\frac{p(1+r_A)^2}
{(1+r_A+r_P)^2-p\big[(1+r_A+r_P)^2 - (1+r_A)^2\big]}.
}
\label{eq:proof_1_1}
\end{equation}

Since $r_P>0$, we have $(1+r_A+r_P)^2 - (1+r_A)^2> 0$. Hence the numerator is strictly increasing
in $p$, while the denominator is strictly decreasing and positive,
implying that $\Delta$ is strictly increasing in $p$.
The limiting values follow directly.
\item \textbf{Non-monotonicity in relative trust.} Recall that
\begin{equation}
\Delta
=
\left(1+\frac{1}{r_A}\right)
\frac{pt}{(1-p)(1+t)^2+p}.
\label{eq:proof_1_2}
\end{equation}

Fix $p\in(0,1)$ and $r_A>0$.
For any $t_2>t_1>0$,
\[
\begin{aligned}
&\frac{pt_2}{(1-p)(1+t_2)^2+p}
-\frac{pt_1}{(1-p)(1+t_1)^2+p}\\
&\quad=
\frac{
pt_2\bigl[(1-p)(1+t_1)^2+p\bigr]
-pt_1\bigl[(1-p)(1+t_2)^2+p\bigr]
}{
\bigl[(1-p)(1+t_1)^2+p\bigr]
\bigl[(1-p)(1+t_2)^2+p\bigr]
}\\
&\quad=
\frac{
p(1-p)\bigl[t_2(1+t_1)^2-t_1(1+t_2)^2\bigr]
+p^2(t_2-t_1)
}{
\bigl[(1-p)(1+t_1)^2+p\bigr]
\bigl[(1-p)(1+t_2)^2+p\bigr]
}\\
&\quad=
\frac{
p(t_2-t_1)\bigl[(1-p)(1-t_1t_2)+p\bigr]
}{
\bigl[(1-p)(1+t_1)^2+p\bigr]
\bigl[(1-p)(1+t_2)^2+p\bigr]
}\\
&\quad=
\frac{
p(t_2-t_1)\bigl[1-(1-p)t_1t_2\bigr]
}{
\bigl[(1-p)(1+t_1)^2+p\bigr]
\bigl[(1-p)(1+t_2)^2+p\bigr]
}.
\end{aligned}
\]

The difference is positive when
$0<t_1<t_2\leq 1/\sqrt{1-p}$,
and negative when
$1/\sqrt{1-p}\leq t_1<t_2$.
Hence, the expression is strictly increasing below
$1/\sqrt{1-p}$ and strictly decreasing above it,
yielding the unique maximizer
\[
t^*=\frac{1}{\sqrt{1-p}}.
\]
Moreover, by~\eqref{eq:proof_1_2}, 
\[
\lim_{t\downarrow0}\Delta=0,
\qquad
\lim_{t\uparrow\infty}\Delta=0.
\]

\item \textbf{Monotonicity in advisor precision, holding relative trust fixed.}
Finally, by~\eqref{eq:proof_1_2}, fixing $(p,t)$, $\Delta$ is strictly decreasing in $r_A>0$.
\end{enumerate}
\end{proof}

\subsection{Proof of Corollary~\ref{cor:loss}}

\begin{proof}
\begin{enumerate}
\item \textbf{{Non-monotonicity in adoption.}} 
Let $A=(1 + r_A)^2$, and $B=(1 + r_A+r_P)^2$.
Then
\[
L^*
=
\frac{p(1-p)r_P^2}{(1-p)B+pA}(r-S_P)^2.
\]

Maximizing $L^*$ with respect to $p$ is equivalent to minimizing
\[
D(p)
=
\frac{(1-p)B+pA}{p(1-p)}
=
\frac{B}{p}+\frac{A}{1-p}.
\]
Differentiating $D(p)$ gives
\[
D'(p)=-\frac{B}{p^2}+\frac{A}{(1-p)^2},
\qquad
D''(p)=\frac{2B}{p^3}+\frac{2A}{(1-p)^3}>0.
\]
Setting $D'(p)=0$ gives
\[
\frac{\sqrt A}{1-p}=\frac{\sqrt B}{p},
\]
and hence
\[
p^*
=\frac{\sqrt B}{\sqrt A+\sqrt B}
=\frac{1+r_A+r_P}{2(1+r_A)+r_P}
=1-\frac{1}{t+2}.
\]
Since $D'$ is strictly increasing and $D'(p^*)=0$,
$D$ is strictly decreasing on $(0,p^*)$ and strictly
increasing on $(p^*,1)$.

Therefore, for $S_P\neq r$, $L^*$ is strictly increasing
on $(0,p^*)$ and strictly decreasing on $(p^*,1)$,
with unique maximizer $p^*$.

For fixed $p\in(0,1)$, the expression for $L^*$ gives
$L^*\to0$ as $t\downarrow0$ and
$L^*\to p(r-S_P)^2$ as $t\uparrow\infty$.

\item \textbf{Relative-trust sufficiency.}
The expression for $L^*$ shows directly that it depends
on $r_A$ and $r_P$ only through $t$.

\item \textbf{Monotonicity in relative trust.} 
Dividing numerator and denominator by $t^2$, we have
\[
L^*
=
\frac{p(1-p)}
{(1-p)\left(\frac{1+t}{t}\right)^2+\frac{p}{t^2}}
(r-S_P)^2.
\]
The denominator is strictly decreasing in $t>0$,
so $L^*$ is strictly increasing in $t$ whenever $S_P\neq r$.

\item \textbf{Implications for source precision.} 
Since $t=r_P/(1+r_A)$ increases in $r_P$
and decreases in $r_A$,
the final monotonicities follow.
\end{enumerate}
\end{proof}

\subsection{Proof of Proposition~\ref{prop:partial-predictability}}

\begin{proof}
Conditional on $H$, the advisor chooses $S_A$ to minimize
\[
\begin{aligned}
L(S_A\mid H)
&=E\!\left[(D-r)^2\mid H\right]\\
&=(1-p)\left(r-\frac{\mu_0+r_AS_A}{1+r_A}\right)^2\\
&\quad+p\,E\!\left[
\left(r-\frac{\mu_0+r_AS_A+r_PS_P}{1+r_A+r_P}\right)^2
\middle|H
\right].
\end{aligned}
\]
Here, consultation occurs with probability $p$ independently of $S_P$ conditional on $H$.

For the consultation term, write
\[
\begin{aligned}
r-\frac{\mu_0+r_AS_A+r_PS_P}{1+r_A+r_P}
={}&r-\frac{\mu_0+r_AS_A+r_PE[S_P\mid H]}
{1+r_A+r_P}\\
&-\frac{r_P}{1+r_A+r_P}
\bigl(S_P-E[S_P\mid H]\bigr).
\end{aligned}
\]
Since
\[
E\!\left[S_P-E[S_P\mid H]\mid H\right]=0,
\]
expanding the square and taking the conditional expectation gives
\[
\begin{aligned}
&E\!\left[
\left(r-\frac{\mu_0+r_AS_A+r_PS_P}{1+r_A+r_P}\right)^2
\middle|H
\right]\\
&\quad=
\left(
r-\frac{\mu_0+r_AS_A+r_PE[S_P\mid H]}
{1+r_A+r_P}
\right)^2\\
&\qquad+
\left(\frac{r_P}{1+r_A+r_P}\right)^2
\operatorname{Var}(S_P\mid H).
\end{aligned}
\]
Therefore,
\[
\begin{aligned}
L(S_A\mid H)
={}&(1-p)\left(r-\frac{\mu_0+r_AS_A}{1+r_A}\right)^2\\
&+p\left(
r-\frac{\mu_0+r_AS_A+r_PE[S_P\mid H]}
{1+r_A+r_P}
\right)^2\\
&+p\left(\frac{r_P}{1+r_A+r_P}\right)^2
\operatorname{Var}(S_P\mid H).
\end{aligned}
\]

The last term does not depend on $S_A$. Hence, the optimization problem is identical to that in Proposition~\ref{prop:optimal-recommendation}, with $S_P$ replaced by $E[S_P\mid H]$. The optimal recommendation is therefore
\[
S_A^*
=
r+\frac{1}{r_A}(r-\mu_0)
+\Delta(p,r_A,r_P)\bigl(r-E[S_P\mid H]\bigr),
\]
where $\Delta(p,r_A,r_P)$ is the counteraction intensity in Proposition~\ref{prop:optimal-recommendation}.

The minimized loss likewise equals the corresponding loss in Proposition~\ref{prop:optimal-recommendation} plus the conditional variance term:
\[
\begin{aligned}
L^*(H)
={}&
\frac{p(1-p)r_P^2}
{(1-p)(1+r_A+r_P)^2+p(1+r_A)^2}
\bigl(r-E[S_P\mid H]\bigr)^2\\
&+
p\left(\frac{r_P}{1+r_A+r_P}\right)^2
\operatorname{Var}(S_P\mid H).
\end{aligned}
\]
\end{proof}

\subsection{Proof of Proposition~\ref{prop:costly-adjustment}}

\begin{proof}
Write the decision loss as
\[
\begin{aligned}
L(S_A)
&=E\!\left[(D-r)^2\right]\\
&=(1-p)\left(\frac{\mu_0+r_AS_A}{1+r_A}-r\right)^2
+p\left(\frac{\mu_0+r_AS_A+r_PS_P}{1+r_A+r_P}-r\right)^2.
\end{aligned}
\]
This is a quadratic function of $S_A$, with
\[
L''(S_A)
=
2(1-p)\left(\frac{r_A}{1+r_A}\right)^2
+2p\left(\frac{r_A}{1+r_A+r_P}\right)^2
=2W.
\]
By Proposition~\ref{prop:optimal-recommendation}, $L'(S_A^*)=0$ and $L(S_A^*)=L^*$. Since $L$ is quadratic, its expansion around $S_A^*$ is exact:
\[
\begin{aligned}
L(S_A)
&=L(S_A^*)+L'(S_A^*)(S_A-S_A^*)
+\frac12L''(S_A^*)(S_A-S_A^*)^2\\
&=L^*+W(S_A-S_A^*)^2.
\end{aligned}
\]
Therefore, the advisor's objective under costly adjustment is
\[
L_\kappa(S_A)
=L^*+W(S_A-S_A^*)^2+\kappa(S_A-S_0)^2.
\]
Differentiating with respect to $S_A$ gives
\[
\frac{\partial L_\kappa}{\partial S_A}
=2W(S_A-S_A^*)+2\kappa(S_A-S_0),
\qquad
\frac{\partial^2 L_\kappa}{\partial S_A^2}
=2(W+\kappa)>0.
\]
Hence, the objective is strictly convex. Setting the first derivative to zero gives
\[
(W+\kappa)S_A=WS_A^*+\kappa S_0.
\]
The unique optimal recommendation is therefore
\[
\begin{aligned}
S_{A,\kappa}^*
&=\frac{WS_A^*+\kappa S_0}{W+\kappa}\\
&=S_A^*+\frac{\kappa}{W+\kappa}(S_0-S_A^*)\\
&=S_A^*+\lambda_\kappa(S_0-S_A^*),
\end{aligned}
\]
where $\lambda_\kappa=\kappa/(W+\kappa)$.

Substituting $S_{A,\kappa}^*$ into the objective gives
\[
\begin{aligned}
L_\kappa^*
&=L^*+
W\left(\frac{\kappa}{W+\kappa}\right)^2(S_A^*-S_0)^2
+\kappa\left(\frac{W}{W+\kappa}\right)^2(S_A^*-S_0)^2\\
&=L^*+\frac{W\kappa^2+\kappa W^2}{(W+\kappa)^2}(S_A^*-S_0)^2\\
&=L^*+\frac{\kappa W}{W+\kappa}(S_A^*-S_0)^2\\
&=L^*+\alpha_\kappa(S_A^*-S_0)^2,
\end{aligned}
\]
where $\alpha_\kappa=\kappa W/(W+\kappa)$.
\end{proof}

\subsection{Proof of Proposition~\ref{prop:private-information-precision}}

\begin{proof} 
Under the decision maker's beliefs, $X,Y,\epsilon_A,\epsilon_P$ are mutually independent. The posterior distribution of $X$ given the advisor's recommendation is
\[
X\mid S_A
\sim
\mathcal N\!\left(
\frac{\sigma_A^2\mu_0+\sigma_X^2S_A}
{\sigma_X^2+\sigma_A^2},
\frac{\sigma_X^2\sigma_A^2}
{\sigma_X^2+\sigma_A^2}
\right).
\]
Since $\theta=X+Y$ and $Y$ is independent of $X,S_A$, the posterior distribution of $\theta$ is
\[
\theta\mid S_A
\sim
\mathcal N\!\left(
\frac{\sigma_A^2\mu_0+\sigma_X^2S_A}
{\sigma_X^2+\sigma_A^2},
\frac{\sigma_X^2\sigma_A^2}
{\sigma_X^2+\sigma_A^2}
+\sigma_Y^2
\right).
\]

Since $\epsilon_P$ is independent of $\theta$ and $S_A$, conditional on $S_A$, the personal-AI recommendation provides an independent Gaussian measurement of $\theta$ with noise variance $\sigma_P^2$. Applying normal updating gives
\[
\begin{aligned}
E[\theta\mid S_A,S_P]
&=
\frac{
\sigma_P^2\,
\frac{\sigma_A^2\mu_0+\sigma_X^2S_A}
{\sigma_X^2+\sigma_A^2}
+
\left(
\frac{\sigma_X^2\sigma_A^2}{\sigma_X^2+\sigma_A^2}
+\sigma_Y^2
\right)S_P
}{
\sigma_P^2+
\frac{\sigma_X^2\sigma_A^2}{\sigma_X^2+\sigma_A^2}
+\sigma_Y^2
}\\[6pt]
&=
\frac{
\sigma_P^2(\sigma_A^2\mu_0+\sigma_X^2S_A)
+
\bigl[\sigma_X^2\sigma_A^2
+\sigma_Y^2(\sigma_X^2+\sigma_A^2)\bigr]S_P
}{
\sigma_P^2(\sigma_X^2+\sigma_A^2)
+\sigma_X^2\sigma_A^2
+\sigma_Y^2(\sigma_X^2+\sigma_A^2)
}\\[6pt]
&=
\frac{
\mu_0+\frac{\sigma_X^2}{\sigma_A^2}S_A
+
\left(
\frac{\sigma_X^2+\sigma_Y^2}{\sigma_P^2}
+\frac{\sigma_X^2\sigma_Y^2}{\sigma_A^2\sigma_P^2}
\right)S_P
}{
1+\frac{\sigma_X^2}{\sigma_A^2}
+\frac{\sigma_X^2+\sigma_Y^2}{\sigma_P^2}
+\frac{\sigma_X^2\sigma_Y^2}{\sigma_A^2\sigma_P^2}
}\\[6pt]
&=
\frac{\mu_0+r_{A,\eta}S_A+r_{P,\eta}S_P}
{1+r_{A,\eta}+r_{P,\eta}},
\end{aligned}
\]
where
\[
r_{A,\eta}=\frac{\sigma_X^2}{\sigma_A^2},
\qquad
r_{P,\eta}
=\frac{\sigma_X^2+\sigma_Y^2}{\sigma_P^2}
+\frac{\sigma_X^2\sigma_Y^2}{\sigma_A^2\sigma_P^2}.
\]
Recall that
\[
\eta=\frac{\sigma_Y^2}{\sigma_X^2+\sigma_Y^2}
=\frac{\sigma_Y^2}{\sigma_0^2}.
\]
Therefore,
\[
\begin{aligned}
r_{A,\eta}
&=\frac{\sigma_X^2}{\sigma_A^2}
=\frac{(1-\eta)\sigma_0^2}{\sigma_A^2}
=(1-\eta)r_A,\\
r_{P,\eta}
&=\frac{\sigma_0^2}{\sigma_P^2}
+\frac{\eta(1-\eta)\sigma_0^4}{\sigma_A^2\sigma_P^2}\\
&=\frac{\sigma_0^2}{\sigma_P^2}
\left[1+\eta(1-\eta)\frac{\sigma_0^2}{\sigma_A^2}\right]\\
&=\bigl[1+\eta(1-\eta)r_A\bigr]r_P.
\end{aligned}
\]
The corresponding relative trust ratio is
\[
t_\eta
=\frac{r_{P,\eta}}{1+r_{A,\eta}}
=
\frac{\bigl[1+\eta(1-\eta)r_A\bigr]r_P}
{1+(1-\eta)r_A}.
\]

When personal AI is not consulted, the posterior mean likewise satisfies
\[
E[\theta\mid S_A]
=
\frac{\sigma_A^2\mu_0+\sigma_X^2S_A}
{\sigma_X^2+\sigma_A^2}
=
\frac{\mu_0+r_{A,\eta}S_A}{1+r_{A,\eta}}.
\]
Thus, both decision rules have the same form as in the baseline model, with $r_A,r_P$ replaced by $r_{A,\eta},r_{P,\eta}$. For $\eta\in[0,1)$, $r_{A,\eta}>0$, so Proposition~\ref{prop:optimal-recommendation} and Corollaries~\ref{cor:counteraction} and~\ref{cor:loss} apply with these substitutions.
\end{proof}

\subsection{Proof of Lemma~\ref{lem:private-information-trust}}

\begin{proof} 
By Proposition~\ref{prop:private-information-precision},
\[
r_{A,\eta}=(1-\eta)r_A,
\qquad
t_\eta
=
\frac{r_P\bigl[1+\eta(1-\eta)r_A\bigr]}
{1+(1-\eta)r_A}.
\]
Since $r_A>0$, $r_{A,\eta}$ is strictly decreasing in $\eta$.

Differentiating $t_\eta$ with respect to $\eta$ gives
\[
\begin{aligned}
\frac{dt_\eta}{d\eta}
&=
\frac{
r_Pr_A(1-2\eta)\bigl[1+(1-\eta)r_A\bigr]
+r_Pr_A\bigl[1+\eta(1-\eta)r_A\bigr]
}{
\bigl[1+(1-\eta)r_A\bigr]^2
}\\
&=
\frac{
r_Pr_A(1-\eta)\bigl[r_A(1-\eta)+2\bigr]
}{
\bigl[1+(1-\eta)r_A\bigr]^2
}.
\end{aligned}
\]
For $r_A,r_P>0$ and $\eta\in[0,1)$, the numerator and denominator are strictly positive, so $dt_\eta/d\eta>0$. Hence, $t_\eta$ is strictly increasing on $[0,1)$.
\end{proof}

\subsection{Proof of Corollary~\ref{cor:private-information-loss}}

\begin{proof} 
Fix $p\in(0,1)$ and $S_P\neq r$. By Proposition~\ref{prop:private-information-precision} and Corollary~\ref{cor:loss}, the advisor's minimized loss is strictly increasing in $t_\eta$. By Lemma~\ref{lem:private-information-trust}, $t_\eta$ is strictly increasing in $\eta$. Therefore, the minimized loss is strictly increasing in $\eta$ on $[0,1)$.
\end{proof}

\section{Trust Investment with a Fixed Recommendation}
\label{app:trust-investment}
The baseline analysis characterizes how the advisor adjusts its recommendation while taking the decision maker's trust parameters as given. In some settings, however, regulatory, ethical, or reputational considerations may limit the extent to which the advisor can freely adjust recommendation content. Even when content is fixed, the advisor may influence how much weight the decision maker places on its recommendation through investments in perceived credibility.

We isolate this strategic margin by fixing the recommendation at $S_A=r$ and allowing the advisor to invest in trust. The fixed-recommendation constraint can be viewed as the limiting case of Section~\ref{subsec:costly-rec-adjust} with $S_0=r$ and $\kappa\to\infty$.

The advisor can either retain its baseline perceived precision $r_A$ or pay a fixed cost $\widetilde c>0$ to raise it to $\widetilde r_A>r_A$. This binary specification captures credibility interventions implemented as discrete choices, such as obtaining certification, introducing an explanation module, or securing a third-party endorsement. Decision makers may also interpret such cues categorically---for example, as certified or uncertified---so a discrete change in perceived trust provides a useful representation of these interventions.

For an advisor precision $q\in\{r_A,\widetilde r_A\}$, the expected decision loss is
\[
\begin{aligned}
L^{\mathrm{fix}}(q)
&=
(1-p)\left(\frac{\mu_0+qr}{1+q}-r\right)^2
+
p\left(\frac{\mu_0+qr+r_PS_P}{1+q+r_P}-r\right)^2\\
&=
(1-p)\frac{(r-\mu_0)^2}{(1+q)^2}
+
p\frac{\bigl[(r-\mu_0)+r_P(r-S_P)\bigr]^2}
{(1+q+r_P)^2}.
\end{aligned}
\]
The advisor compares $L^{\mathrm{fix}}(r_A)$, the loss without investment, with $L^{\mathrm{fix}}(\widetilde r_A)+\widetilde c$, the total loss under investment.

\textbf{Proposition (Trust investment threshold).} Investment is optimal if and only if $\widetilde c\leq\overline c$, where
\[
\begin{aligned}
\overline c
&=L^{\mathrm{fix}}(r_A)-L^{\mathrm{fix}}(\widetilde r_A)\\
&=(1-p)(r-\mu_0)^2
\left[
\frac{1}{(1+r_A)^2}
-\frac{1}{(1+\widetilde r_A)^2}
\right]\\
&\quad+
p\bigl[(r-\mu_0)+r_P(r-S_P)\bigr]^2
\left[
\frac{1}{(1+r_A+r_P)^2}
-\frac{1}{(1+\widetilde r_A+r_P)^2}
\right].
\end{aligned}
\]
The threshold $\overline c$ is affine in $p$. It is strictly increasing in $p$ when the reduction in decision loss from investment is larger in the consultation state than in the no-consultation state, strictly decreasing when it is smaller, and constant when the two reductions are equal.

\begin{proof}
Investment is optimal precisely when
\[
L^{\mathrm{fix}}(\widetilde r_A)+\widetilde c
\leq L^{\mathrm{fix}}(r_A),
\]
or equivalently,
\[
\widetilde c
\leq L^{\mathrm{fix}}(r_A)-L^{\mathrm{fix}}(\widetilde r_A)
=\overline c.
\]
Substituting the expression for $L^{\mathrm{fix}}$ gives the stated threshold. The two state-specific reductions in loss do not depend on $p$, so their weighted average is affine in $p$, with slope equal to the consultation-state reduction minus the no-consultation-state reduction.
\end{proof}

The threshold $\overline c$ measures the maximum cost the advisor is willing to pay for the increase in perceived precision. Higher personal-AI adoption shifts weight toward the consultation state, but does not necessarily increase the return to trust investment. The direction depends on how much additional advisor trust improves the decision in each state.

For example, if $\mu_0=r$ but $S_P\neq r$, the no-consultation decision already equals the target. Trust investment then has value only when personal AI is consulted, and the threshold increases with $p$. Conversely, if
\[
\frac{\mu_0+r_PS_P}{1+r_P}=r
\qquad\text{and}\qquad
\mu_0\neq r,
\]
the consultation-state decision already equals the target. Trust investment then has value only in the no-consultation state, and the threshold decreases with $p$. Thus, when recommendation content is constrained, personal-AI adoption can either strengthen or weaken the advisor's incentive to invest in credibility.

\section{Measurement and Estimation}
\label{subsec:measurement-estimation}

The experiment reports consultation, forecast revision, and absolute forecast error. Within-round error reduction is
\[
|y_t-x_{0,t}|-|y_t-x_{1,t}|,
\]
where $x_{0,t}$ and $x_{1,t}$ are the initial and final forecasts and $y_t$ is the realized outcome. Positive values indicate that the final forecast is closer to the outcome. This comparison describes changes within a round; it does not identify an advice effect relative to a no-advice condition. Confidence is elicited within rounds; the post-game survey records stated perceptions of the advice sources.

The relationships below connect observable forecast revisions to the model's advice weights under linear aggregation. The fitted coefficients describe average aggregation under the specified linear rule; their perceived-precision interpretation is conditional on that rule. To describe the relationships, we suppress the participant index. The initial forecast $x_{0,t}$ and final forecast $x_{1,t}$ correspond to the model's prior mean $\mu_0$ and decision $D$. The recommendations $S_{P,t}$ and $S_{E,t}$ denote personal AI and enterprise AI. The enterprise-AI precision $r_E$ corresponds to the focal advisor's precision $r_A$ in the theory.

\subsection{Single-Source Advice}

For a round in which only personal-AI advice is observed and $S_{P,t}\neq x_{0,t}$, the standard weight-of-advice measure is
\[
WOA_{P,t}=\frac{x_{1,t}-x_{0,t}}{S_{P,t}-x_{0,t}}.
\]
It measures the fraction of the distance from the initial forecast to the recommendation covered by the revision. Values of zero and one correspond to no revision and full adoption of the recommendation, respectively. Values between zero and one indicate partial adjustment toward the recommendation. Negative values and values above one are possible, and the ratio can be unstable when advice is close to the initial forecast.

Under the single-source aggregation rule,
\[
D_P=\frac{\mu_0+r_P S_P}{1+r_P}
=(1-w_P)\mu_0+w_P S_P,
\qquad w_P=\frac{r_P}{1+r_P}.
\]
Thus, WOA equals $w_P$ when the observed forecasts satisfy this rule. For $0\leq w_P<1$, the inverse relationship is
\[
r_P=\frac{w_P}{1-w_P}.
\]
This ratio is personal AI's perceived precision relative to the prior. It is distinct from $t$, which also accounts for the focal advisor's influence. If $\widehat w_P$ denotes a mean WOA in $[0,1)$, the corresponding transformation is
\[
\widehat r_P^{\mathrm{WOA}}=\frac{\widehat w_P}{1-\widehat w_P}.
\]
Because the transformation is nonlinear, this quantity is not the mean of the observation-level precision ratios. A weight of one corresponds to an unbounded precision ratio, and weights outside $[0,1]$ do not admit a nonnegative perceived-precision interpretation.

A constrained linear aggregation model provides an alternative representation:
\begin{align*}
x_{1,t}&=w_{0P}x_{0,t}+w_PS_{P,t}+\varepsilon_t,\\
&w_{0P}+w_P=1,\qquad w_{0P}\geq0,\quad w_P\geq0.
\end{align*}
Here, $w_{0P}$ and $w_P$ are the weights on the initial forecast and personal AI. If $\widehat w_{0P}>0$, the implied perceived-precision ratio is
\[
\widehat r_P^{\mathrm{reg}}
=\frac{\widehat w_P}{\widehat w_{0P}}
=\frac{\widehat w_P}{1-\widehat w_P}.
\]
The same single-source relationships apply to enterprise AI or human-expert advice when each is the only source observed.

\subsection{Multiple-Source Advice}

When enterprise AI and personal AI are both observed, the corresponding constrained model is
\begin{align*}
x_{1,t}&=w_0x_{0,t}+w_ES_{E,t}+w_PS_{P,t}+\varepsilon_t,\\
&w_0+w_E+w_P=1,\qquad w_0\geq0,\quad w_E\geq0,\quad w_P\geq0.
\end{align*}
For $\widehat w_0>0$, the implied perceived precisions are
\[
\widehat r_E=\frac{\widehat w_E}{\widehat w_0},
\qquad
\widehat r_P=\frac{\widehat w_P}{\widehat w_0}.
\]
Relative trust can then be written as
\[
\widehat t
=\frac{\widehat r_P}{1+\widehat r_E}
=\frac{\widehat w_P}{\widehat w_0+\widehat w_E}
=\frac{\widehat w_P}{1-\widehat w_P}.
\]
The direct weight ratio remains defined when $\widehat w_0=0$ provided that $\widehat w_0+\widehat w_E>0$, although the separate perceived-precision ratios are then not finite. The single-source and multiple-source transformations have the same algebraic form but describe different objects: $r_P$ compares personal AI with the prior, whereas $t$ compares it with the prior and focal advisor together.

These interpretations require the linear aggregation model and sufficient independent variation in advice relative to initial forecasts. In particular, separating the two advice weights requires that their forecast gaps are not collinear. The experiment includes two single-source rounds per source, limiting participant-specific estimation. We estimate pooled coefficients rather than participant-specific weights. For the 138 dual-source rounds, the two forecast-gap columns have full rank.

\subsection{Estimation and Weight-of-Advice Sensitivity}
\label{subsec:single-source-robustness}

The single-source estimator minimizes $\sum_{i,t}[x_{1,it}-x_{0,it}-w(S_{it}-x_{0,it})]^2$ over $w\in[0,1]$, separately for each source. The dual-source estimator minimizes the corresponding squared residuals over $w_E,w_P\geq0$ and $w_E+w_P\leq1$. Each round receives equal weight in these objectives. The single-source unconstrained solution is a weighted average of round-level WOA values with weights proportional to $(S_{it}-x_{0,it})^2$, before imposing the coefficient constraint.

Uncertainty intervals for the pooled weights use 20,000 bootstrap samples of participants, retaining each sampled participant's rounds and refitting the constrained model. Percentile intervals use the 2.5th and 97.5th percentiles. The relative-influence index is recomputed from each bootstrap fit. The dual-source point estimates imply $\widehat r_E=1.07$ and $\widehat r_P=0.80$ under the model's mapping.

For the supplementary WOA summary, we retain rounds with $|S_{it}-x_{0,it}|\geq5$, truncate the ratios to $[0,1]$, and average first within participants and then across participants. This retains 138 enterprise-AI rounds, 138 personal-AI rounds, and 135 human-expert rounds, with all 69 participants represented for every source. Figure~\ref{fig:single-source-weights} compares these summaries with the pooled estimates. Table~\ref{tab:woa-sensitivity} reports alternatives to the gap threshold and truncation rule.

\begin{figure}[!htbp]
\centering
\includegraphics[width=0.82\textwidth]{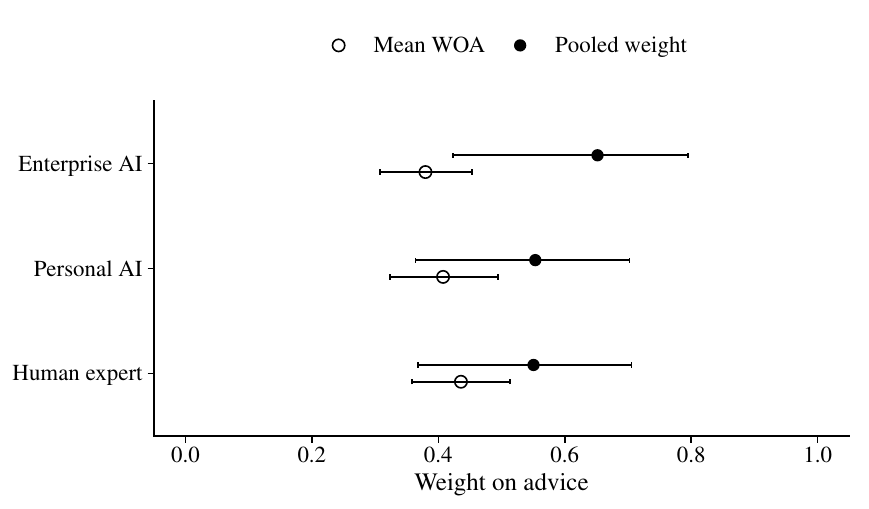}
\caption{Single-source advice weights under two summaries. Filled points report pooled constrained least-squares weights. Open points report participant-averaged WOA after retaining advice gaps of at least five units and truncating ratios to $[0,1]$. Intervals are 95\% participant bootstrap percentile intervals using 20,000 samples.}
\label{fig:single-source-weights}
\end{figure}

\begin{table}[!htbp]
\small\centering
\caption{Sensitivity of Participant-Averaged WOA}
\label{tab:woa-sensitivity}
\begin{tabular}{llccc}
\hline
Minimum absolute gap & Ratio treatment & Enterprise AI & Personal AI & Human expert \\
\hline
Nonzero & Untruncated & 0.38 & 0.41 & 0.44 \\
Nonzero & Truncated to $[0,1]$ & 0.38 & 0.41 & 0.42 \\
1 & Untruncated & 0.38 & 0.41 & 0.44 \\
1 & Truncated to $[0,1]$ & 0.38 & 0.41 & 0.42 \\
5 & Untruncated & 0.38 & 0.41 & 0.45 \\
5 & Truncated to $[0,1]$ & 0.38 & 0.41 & 0.44 \\
10 & Untruncated & 0.39 & 0.41 & 0.45 \\
10 & Truncated to $[0,1]$ & 0.39 & 0.41 & 0.43 \\
\hline
\end{tabular}
\begin{minipage}{0.92\textwidth}\vspace{0.5em}\footnotesize
Notes. Thresholds are inclusive, and zero denominators are always excluded. Ratios are averaged within participant before averaging across participants with at least one eligible round. All conditions retain 69 participants except the human-expert condition at a ten-unit threshold, which retains 68. Source configurations occur on different scenarios in a fixed order.
\end{minipage}
\end{table}

\begin{samepage}
Paired comparisons of the five-unit, truncated WOA summaries give differences of $-0.028$ for enterprise versus personal AI ($p=0.379$), $-0.056$ for enterprise AI versus the human expert ($p=0.178$), and $-0.028$ for personal AI versus the human expert ($p=0.478$). These are unadjusted exploratory comparisons across source configurations. They do not establish equivalence or identify a source-label effect.\par
\end{samepage}

\subsection{Consultation Counts and Access Outcomes}
\label{subsec:consultation-benchmark}

Table~\ref{tab:block2-access} reports the access rates, separating the two time limits used with ten-character codes. Successful access equals advice revelation in the code-required rounds.

\begin{table}[!htbp]
\small\centering
\caption{Personal-AI Access in Block 2}
\label{tab:block2-access}
\begin{tabular}{lrrrrr}
\hline
Code length & Time limit & Observations & Attempts & Advice revealed & Revealed (\%) \\
\hline
0 & None & 138 & -- & 138 & 100.0 \\
4 & 10 seconds & 138 & 74 & 70 & 50.7 \\
6 & 10 seconds & 138 & 68 & 60 & 43.5 \\
8 & 10 seconds & 138 & 63 & 55 & 39.9 \\
10 & 10 seconds & 138 & 59 & 41 & 29.7 \\
10 & 8 seconds & 138 & 58 & 27 & 19.6 \\
\hline
\end{tabular}
\begin{minipage}{0.92\textwidth}\vspace{0.5em}\footnotesize
Notes. Every condition contains two rounds from each of 69 participants. In zero-code rounds advice is displayed automatically. Conditions occur in the order shown. The pooled ten-character condition has 276 observations, 117 attempts, and 68 successful accesses, corresponding to attempt and revelation rates of 42.4\% and 24.6\%.
\end{minipage}
\end{table}

For code length $\ell$, let $K_{i\ell}$ be the participant's consultation count across $n_\ell$ rounds. The homogeneous independent-round benchmark is
\[
K_{i\ell}\sim\mathrm{Binomial}(n_\ell,p_\ell).
\]
We estimate $p_\ell$ by the sample consultation rate and compare the observed count frequencies with the fitted binomial frequencies using Pearson's statistic. Counts of three and four consultations are pooled for ten-character codes.

\begin{table}[!htbp]
\small\centering
\caption{Consultation Counts and the Fitted Binomial Benchmark}
\label{tab:block2-binomial-gof}
\begin{tabular}{lrrllr}
\hline
Code length & Rounds & Rate & Observed counts & Expected counts & $\chi^2$ \\
\hline
4 & 2 & 0.507 & $(28,12,29)$ & $(16.8,34.5,17.8)$ & 29.3 \\
6 & 2 & 0.435 & $(33,12,24)$ & $(22.0,33.9,13.0)$ & 28.8 \\
8 & 2 & 0.399 & $(34,15,20)$ & $(25.0,33.1,11.0)$ & 20.6 \\
10 & 4 & 0.246 & $(42,6,8,13)$ & $(22.3,29.1,14.3,3.4)$ & 66.2 \\
\hline
\end{tabular}
\begin{minipage}{0.92\textwidth}\vspace{0.5em}\footnotesize
Notes. Counts are for $K=0,1,2$ at lengths four, six, and eight, and $K=0,1,2,\geq3$ at length ten. Nominal degrees of freedom are 1, 1, 1, and 2. We simulate 200,000 samples of 69 participants from each fitted binomial distribution, re-estimate its probability using the unpooled counts in every sample, apply the same pooling rule, and recompute the statistic. No simulated statistic reaches the observed value in any condition. The plus-one Monte Carlo estimate is $1/200001\approx0.000005$ in each case; this is a simulation estimate, not an upper bound on the tail probability. The pooled upper expected count at length ten is below five, motivating simulation rather than reliance on the chi-square approximation. Comparisons are exploratory and unadjusted.
\end{minipage}
\end{table}

Figure~\ref{fig:access-outcomes}(b) reports the overall distribution across participants; Table~\ref{tab:block2-binomial-gof} reports the count distributions within code lengths. Differences across participants and dependence across their repeated decisions can both produce departures from the fitted benchmark.

\begin{samepage}
\subsection{Counteraction Comparison}
\label{subsec:counteraction-analysis}

The comparison uses the third group of Block 3 and includes only rounds with personal AI available. The group assigns each participant two counteracting rounds, one unadjusted round with personal AI, and one round without personal AI, shuffled across four scenarios. The first two groups contain no unadjusted enterprise advice alongside personal AI and are excluded from the policy comparison.\par
\end{samepage}

\begin{figure}[!htbp]
\centering
\includegraphics[width=0.84\textwidth]{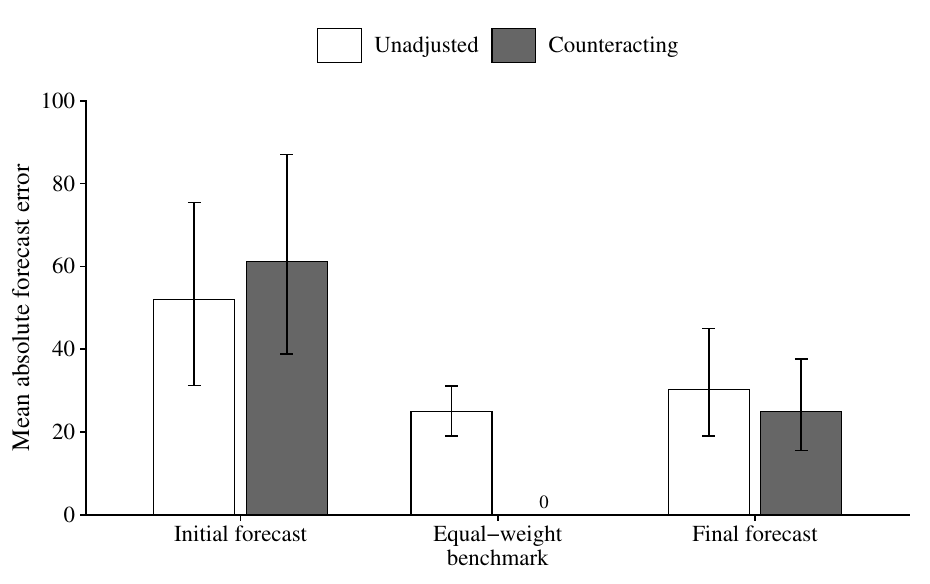}
\caption{Forecast error under counteracting and unadjusted advice. The comparison uses the third group of Block 3, restricted to rounds with personal AI available: two counteracting rounds and one unadjusted round per participant. Bars are means of participant-specific condition averages; intervals are 95\% participant bootstrap intervals. The equal-weight benchmark averages the two displayed recommendations and is a constructed benchmark, not a participant forecast. Its error under counteraction is zero by construction.}
\label{fig:block3-counteraction}
\end{figure}

For each participant, we subtract absolute error in the unadjusted round from average absolute error in the two counteracting rounds. The mean difference is $-5.46$, with a paired standard deviation of 41.55. The main-text interval uses the $t$ distribution with 68 degrees of freedom. A two-sided paired $t$ test gives $p=0.279$; an exploratory signed-rank test gives $p=0.373$. The two tests concern different features of the paired distribution. Thirty-seven of 69 participants have a lower condition-average error under counteraction. Figure~\ref{fig:counteraction-distribution} shows the complete distribution of differences.

\begin{table}[!htbp]
\small\centering
\caption{Sensitivity of the Counteraction Comparison}
\label{tab:counteraction-specifications}
\begin{tabular}{lrrr}
\hline
Specification & Difference & Standard error & $p$-value \\
\hline
Paired participant means & $-5.46$ & 5.00 & 0.279 \\
Participant fixed effects & $-5.46$ & 6.13 & 0.376 \\
Participant and scenario fixed effects & $-7.27$ & 5.85 & 0.218 \\
\hline
\end{tabular}
\begin{minipage}{0.92\textwidth}\vspace{0.5em}\footnotesize
Notes. Differences are counteracting minus unadjusted mean absolute final error. Regressions use 207 rounds from 69 participants and participant-clustered HC1 standard errors, including the finite-sample correction for the number of estimated coefficients. Reported $p$-values use 68 degrees of freedom. The first row uses the standard error of the 69 paired differences. Different finite-sample corrections account for the standard-error difference between the first two rows.
\end{minipage}
\end{table}

\begin{figure}[!htbp]
\centering
\includegraphics[width=0.86\textwidth]{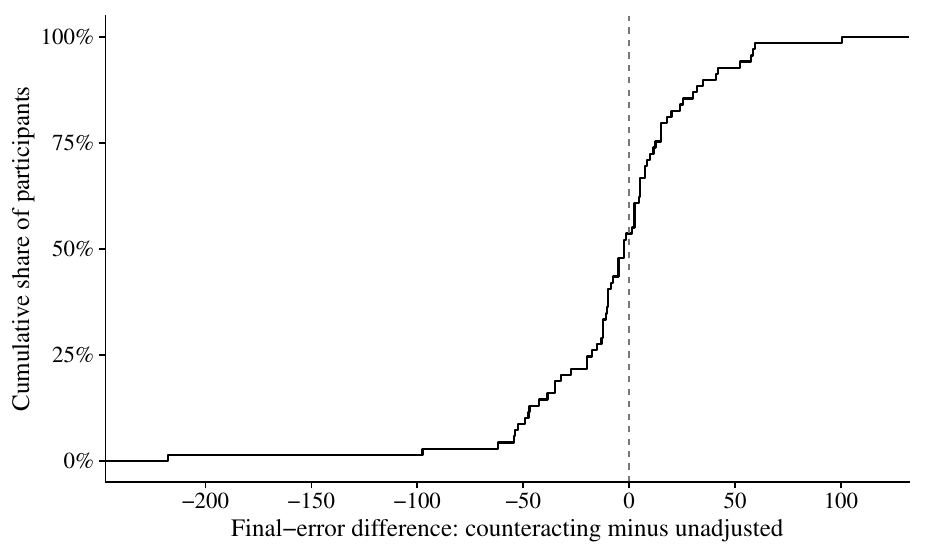}
\caption{Distribution of participant-level differences in final error. Each observation subtracts a participant's unadjusted-round error from their average error across two counteracting rounds in the third group of Block 3. The empirical cumulative distribution includes all 69 differences, ranging from $-217.5$ to $100.5$. The vertical reference marks zero.}
\label{fig:counteraction-distribution}
\end{figure}

Applying the Block 1 pooled weights to the counteracting rounds gives a predicted mean absolute error of 22.16, compared with observed error of 24.88. The retained-initial-judgment component has mean absolute magnitude 21.28, while the unequal-advice-weight component has magnitude 5.26. Their absolute magnitudes are not additive because the components can reinforce or offset one another. These calculations hold the Block 1 coefficients fixed; they are an illustration of the aggregation identity rather than an estimated decomposition of observed errors. The predicted forecasts differ from observed final forecasts by 30.20 units on average in absolute value.

With both recommendations observed, the Block 1 point estimates would cancel the personal-AI term at intensity $\widehat w_P/\widehat w_E=0.75$. Its 95\% bootstrap interval is $[0.46,1.29]$, which includes the implemented intensity of one. This cancellation ratio is conditional on the fitted aggregation rule and is distinct from the optimal policy under uncertain consultation.

\subsection{Sample Sensitivity}
\label{subsec:sample-sensitivity}

The main analyses use the common 69-participant sample. Table~\ref{tab:sample-sensitivity} repeats the principal summaries using all 74 complete participant records, including five additional participants. The sample definitions are held constant across blocks within each analysis.

\begin{table}[!htb]
\small\centering
\caption{Sensitivity to Including Five Additional Participants}
\label{tab:sample-sensitivity}
\begin{tabular}{lcc}
\hline
Summary & 69 participants & 74 participants \\
\hline
Single-source enterprise weight & 0.65 & 0.70 \\
Single-source personal-AI weight & 0.55 & 0.58 \\
Single-source human-expert weight & 0.55 & 0.56 \\
Dual-source initial-forecast weight & 0.35 & 0.35 \\
Dual-source enterprise weight & 0.37 & 0.35 \\
Dual-source personal-AI weight & 0.28 & 0.30 \\
Relative-influence index $\widehat t$ & 0.39 & 0.43 \\
Counteraction difference & $-5.46$ & $-5.48$ \\
95\% paired interval & $[-15.44,4.52]$ & $[-15.00,4.04]$ \\
Advice revealed, 4-character codes (\%) & 50.7 & 51.4 \\
Advice revealed, 6-character codes (\%) & 43.5 & 45.9 \\
Advice revealed, 8-character codes (\%) & 39.9 & 41.9 \\
Advice revealed, 10-character codes (\%) & 24.6 & 26.4 \\
\hline
\end{tabular}
\begin{minipage}{0.92\textwidth}\vspace{0.5em}\footnotesize
Notes. Both columns use the same estimators and round restrictions. The counteraction difference compares participant-specific mean absolute final errors in the third group of Block 3, conditional on personal-AI availability. Intervals use the paired $t$ distribution. Ten-character access rates pool both time limits.
\end{minipage}
\end{table}

\end{APPENDICES}

\end{document}